\documentclass{article}

\usepackage{amsfonts}
\usepackage{amsmath}
\usepackage{amsthm}
\usepackage{xcolor}
\usepackage{comment}
\usepackage{wrapfig}
\usepackage{float}
\usepackage{placeins}
\usepackage{bbm}
\usepackage{microtype}
\usepackage{graphicx}
\usepackage{subfigure}
\usepackage{booktabs} 

\usepackage{hyperref}

\newcommand{\cov}{\mathbb{C}\text{ov}}
\newcommand{\var}{\mathbb{V}\text{ar}}

\newcommand{\G}{\mathcal{G}}
\newcommand{\F}{\mathcal{F}}

\newcommand{\X}{\mathcal{X}}
\newcommand{\Y}{\mathcal{Y}}
\newcommand{\E}{\mathbb{E}}
\newcommand{\R}{\mathbb{R}}

\makeatletter
\newcommand*{\indep}{%
  \mathbin{%
    \mathpalette{\@indep}{}%
  }%
}
\newcommand*{\nindep}{%
  \mathbin{
    \mathpalette{\@indep}{\not}
  }%
}
\newcommand*{\@indep}[2]{%
  \sbox0{$#1\perp\m@th$}
  \sbox2{$#1=$}
  \sbox4{$#1\vcenter{}$}
  \rlap{\copy0}
  \dimen@=\dimexpr\ht2-\ht4-.2pt\relax
  \kern\dimen@
  {#2}%
  \kern\dimen@
  \copy0 
}

\newtheorem{theorem}{Theorem}
\newtheorem{lemma}[theorem]{Lemma}
\newtheorem{definition}[theorem]{Definition}

\newtheorem{corollary}[theorem]{Corollary}

\usepackage[accepted]{icml2020}

\begin{document}

\twocolumn[
\icmltitle{Robust Learning with the Hilbert-Schmidt Independence Criterion}



\icmlsetsymbol{equal}{*}

\begin{icmlauthorlist}
\icmlauthor{Daniel Greenfeld}{technion}
\icmlauthor{Uri Shalit}{technion}

\end{icmlauthorlist}

\icmlaffiliation{technion}{Technion Institute of
Technology, Haifa, Israe}

\icmlcorrespondingauthor{Daniel Greenfeld}{danielgreenfeld3@gmail.com}
\icmlcorrespondingauthor{Uri Shalit}{urishalit@technion.ac.il}

\icmlkeywords{Machine Learning, ICML}

\vskip 0.3in
]



\printAffiliationsAndNotice{}  

\begin{abstract}
\label{abstract}
We investigate the use of a non-parametric independence measure, the Hilbert-Schmidt Independence Criterion (HSIC), as a loss-function for learning robust regression and classification models. This loss-function encourages learning models where the distribution of the residuals between the label and the model prediction is statistically independent of the distribution of the instances themselves. This loss-function was first proposed by \citet{mooij2009regression} in the context of learning causal graphs. We adapt it to the task of learning for unsupervised covariate shift: learning on a source domain without access to any instances or labels from the unknown target domain, but with the assumption that $p(y|x)$  (the conditional probability of labels given instances) remains the same in the target domain. We show that the proposed loss is expected to give rise to models that generalize well on a class of target domains characterised by the complexity of their description within a reproducing kernel Hilbert space. Experiments on unsupervised covariate shift tasks  demonstrate that models learned with the proposed loss-function outperform models learned with standard loss functions, achieving state-of-the-art results on a challenging cell-microscopy unsupervised covariate shift task.
\end{abstract}

\section{Introduction}
\label{Introduction}
In recent years there has been much interest in methods for learning \emph{robust models}: models that are learned using certain data but perform well even on data drawn from a distribution which is different from the training distribution. This interest stems from demand for models which can perform under conditions of transfer learning and domain adaptation \citep{rosenfeld2018elephant}. This is especially relevant as training sets such as labeled image collections are often restricted to a certain setting, time or place, while the learned models are expected to generalize to cases which are beyond the specifics of how the training data was collected.

More specifically, we consider the following learning problem, called \emph{unsupervised covariate shift}. Let $(X,Y)$ be a pair of random variables such that $X \in \X$ and  $Y\in \Y \subset \R$,  with a joint distribution $P_{\text{source}}(X,Y)$, such that $X$ are the instances and $Y$ the labels. Our goal is, given a training set drawn from $P_{\text{source}}(X,Y)$, to learn a model predicting $Y$ from $X$ that works well on a different, a-priori unknown target distribution $P_{\text{target}}(X,Y)$.
In a covariate shift scenario, the assumption is that $P_{\text{target}}(Y\mid X) = P_{\text{source}}(Y\mid X)$ but the marginal distribution $P_{\text{target}}(X)$ can change between source and target. We focus on \emph{unsupervised} covariate shift, where we have no access to samples $X$ or $Y$ from the target domain.

In this paper we propose using a loss function inspired by work in the causal inference community. 
We consider a model in which the relation between the instance $X$ and its label $Y$ is of the form:
\begin{align}\label{eq:model}
    Y=f^\star(X) + \varepsilon,\quad \varepsilon \indep X,
\end{align}
where the variable $\varepsilon$ denotes noise which is independent of the distribution of the random variable $X$. 

Given a well-specified model family and enough samples, one can learn $f^\star$, in which case there is no need to worry about covariate shift. However, in many realistic cases we cannot expect to have the true model in our model class, nor can we expect to have enough samples to learn the true model even if it is in our model class. While traditional methods rely on unlabeld data from the target domain to reason about $P_{\text{target}}(X)$, throughout this work we do not assume that we have \textit{any} samples from a test distribution, nor that the model is well-specified.

The basic idea presented in this paper is as follows: by Eq. \eqref{eq:model} we have $Y-f^\star(X) \indep X$. Standard loss functions aim to learn a model $\hat{f}$ such that $\hat{f}(X) \approx Y$, or such that $\hat{p}(Y|X)$ is high. We follow a different approach: learning a model $\hat{f}$ such that $Y - \hat{f}(X) $ is approximately independent of the distribution of $X$. Specifically, we propose  measuring independence using the Hilbert Schmidt Independence Criterion (HSIC): a non-parametric method that does not assume a specific noise distribution for $\varepsilon$ \citep{gretton2005measuring,gretton2008kernel}. 
This approach was first proposed by \citet{mooij2009regression} in the context of causal inference. As \citet{mooij2009regression} point out, this approach can be contrasted with learning with loss functions such as the squared-loss or absolute-loss, which implicitly assume that $\varepsilon$ has, respectively, a Gaussian or Laplacian distribution.

Intuitively, covariate shift is most harmful when the target distribution has more mass on areas of $\mathcal{X}$ on which the learned model performs badly. Thus, being robust against unsupervised covariate shift means having no certain sub-population (of positive measure) on which the model performs particularly badly. This of course might come at a certain cost to the performance on the known source distribution.

The following toy example showcases that standard loss functions do not necessarily incentivize such behaviour. Suppose $\mathcal{X}=\{0,1\}$, and $Y=X$. Let $P(X=0)=\varepsilon$, and consider the following hypothesis class: $\mathcal{H}=\{h_1,h_2\}$, where $h_1(x)=1$ for all $x$, and $h_2(x)=x-0.01$. For small values of $\varepsilon$, an algorithm minimizing the mean squared error (or any standard loss) will output $h_1$. An algorithm relying on the independence criteria on the other hand will output $h_2$, since its residuals are independent of $X$. Which is preferable? That depends on the target (test) distribution. If it is the same as the training distribution, $h_1$ is indeed a better choice. However, if we do not know the target distribution, then $h_2$ might be the better choice since it will always incur relatively small loss, as opposed to $h_1$ which might have very poor performance, say if $P(X=0)$ and $P(X=1)$ are switched during test time.
While pursuing robustness against any change in $P(X)$ is difficult, we will show below that learning with the HSIC-loss provides a natural trade-off between the generalization guarantees on the unknown target, and the complexity of the changes in the target relative to the source distribution.

Our contributions relative to the first proposal by \citet{mooij2009regression} are as follows:
\begin{enumerate}
    \item We prove that the HSIC objective is learnable for an hypothesis class of bounded Rademacher complexity.
    \item We prove that minimizing the HSIC-loss minimizes a worst-case loss over a class of unsupervised covariate shift tasks.
    \item We provide experimental validation using both linear models and deep networks, showing that learning with the HSIC-loss is competitive on a variety of unsupervised covariate shift benchmarks. 
    \item We provide code, including a PyTorch \citep{paszke2019pytorch} class for the HSIC-loss\footnote{\href{https://github.com/danielgreenfeld3/XIC}{https://github.com/danielgreenfeld3/XIC}.}.
\end{enumerate}

\section{Background and Setup}
\label{Background}
The  Hilbert-Schmidt independence criterion (HSIC), introduced by \citet{gretton2005measuring,gretton2008kernel}, is a useful method for testing if two random variables are independent. We give its basics below. 

The root of the idea is that while $\cov (A,B)=0$  does not imply that two random variables $A$ and $B$ are independent, having $\cov (s(A),t(B))=0$ for all bounded continuous functions $s$ and $t$ \textbf{does} actually imply independence \citep{renyi1959measures}. Since going over all bounded continuous functions is not tractable, \citet{gretton2005kernel} propose evaluating $\sup_{s\in \F, t\in \G} \cov\left[s(x), t(y) \right]$ where $\F,\G$ are universal Reproducing Kernel Hilbert Spaces (RKHS). This allows for a tractable computation and is equivalent in terms of the independence property. \citet{gretton2005measuring} then introduced HSIC as an upper bound to the measure introduced by \citet{gretton2005kernel}, showing it has superior performance and is easier to work with statistically and algorithmically. 

\subsection{RKHS Background}
A reproducing kernel Hilbert space $\F$ is a Hilbert space of functions from $\mathcal{X}$ to $\mathbb{R}$ with the following (reproducing) property: there exist a positive definite kernel $K:\mathcal{X}\times\mathcal{X} \to \mathbb{R}$ and a mapping function $\phi$ from $\mathcal{X}$ to $\F$ s.t. $K(x_1,x_2)=\langle\phi(x_1),\phi(x_2)\rangle_\F$.
Given two separable (having a complete orthonormal basis) RKHSs $\F$ and $\G$ on metric spaces $\mathcal{X}$ and $\mathcal{Y}$, respectively, and a linear operator $C:\F\to\G$, the Hilbert-Schmidt norm of $C$ is defined as follows:
\begin{align*}
    \|C\|_{\text{HS}}=\sum_{i,j}\langle Cu_i,v_j\rangle_\G^2,
\end{align*}
where $\{u_i\}$ and $\{u_i\}$ are some orthonormal basis for $\F$ and $\G$ respectively. Here we consider two probability spaces $\mathcal{X}$ and $\mathcal{Y}$ and their corresponding RKHSs $\F$ and $\G$. The mean elements $\mu_x$ and $\mu_y$ are defined such that $\langle \mu_x,s \rangle_\F := \E[\langle \phi(x),s\rangle_\F] = \E[s(x)]$, and likewise $\langle \mu_y,t \rangle_\G := \E[\langle \psi(y),t\rangle_\G] = \E[t(y)]$, where $\psi$ is the embedding from $\mathcal{Y}$ to $\G$. Notice that we can compute the norms of those operators quite easily: $\|\mu_x\|_\F^2=\E[K(x_1,x_2)]$ where the expectation is done over i.i.d. samples of pairs from $\mathcal{X}$.
For $s\in\F$ and $t\in\G$, their tensor product $s\otimes t:\G\to\F$ is defined as follows:
    $(s\otimes t)(h)= \langle t,h \rangle_\G \cdot s.$
The Hilbert-Schmidt norm of the tensor product can be shown to be given by $\|s\otimes t\|_\text{HS}^2=\|s\|_\F^2\cdot\|t\|_\G^2$.
Equipped with these definitions, we are ready to define the cross covariance operator $C_{xy}:\G\to\F$:
\begin{align*}
    C_{xy}=\E[\phi(x)\otimes\psi(y)]-\mu_x\otimes\mu_y.
\end{align*}

\subsection{HSIC}
Consider two random variables $X$ and $Y$, residing in two metric spaces $\mathcal{X}$ and $\mathcal{Y}$ with a joint distribution on them, and two separable RKHSs $\F$ and $\G$ on $\mathcal{X}$ and $\mathcal{Y}$ respectively. HSIC is defined as the Hilbert Schmidt norm of the cross covariance operator:
\begin{align*}
    \text{HSIC}(X,Y;\F,\G)\equiv \|C_{xy}\|_\text{HS}^2.
\end{align*}

\citet{gretton2005measuring} show that:
\begin{align}\label{eq:hsiccoco}
     \text{HSIC}(X,Y;\F,\G) \geq  \sup_{s\in \F, t\in \G} \cov\left[s(x), t(y) \right],
\end{align}
an inequality which we use extensively for our results.

We now state Theorem 4 of \citet{gretton2005measuring}) which shows the properties of HSIC as an independence test:
\begin{theorem}[\citet{gretton2005measuring}, Theorem 4]  \label{thm:1}
Denote by $\F$ and $\G$ RKHSs both with universal kernels, $k,l$ respectively on compact domains $\mathcal{X}$ and $\mathcal{Y}$. Assume without loss of generality that $\|s\|_\infty \le 1$ for all $s\in\F$ and likewise $\|t\|_\infty \le 1$ for all $t\in\G$.

Then the following holds: $\|C_{xy}\|_\text{HS}^2 = 0 \Leftrightarrow X\indep Y$.
\end{theorem}

Let $\{(x_i,y_i)\}_{i=1}^n$ be i.i.d. samples from the joint distribution on $\X \times \Y$. The empirical estimate of HSIC is given by:
\begin{align}\label{eq:emphsic}
    \widehat{\text{HSIC}}\{(x_i,y_i)\}_{i=1}^n;\F,\G) = \frac{1}{(n-1)^2} \textbf{tr} KHLH,
\end{align}
where $K_{i,j}=k(x_i,x_j)$, $L_{i,j}=l(y_i,y_j)$ are kernel matrices for the kernels $k$ and $l$ respectively, and $H_{i,j}=\delta_{i,j}-\frac{1}{n}$ is a centering matrix.
The main result of \citet{gretton2005measuring} is that the empirical estimate $\widehat{\text{HSIC}}$ converges to HSIC at a rate of $O\left(\frac{1}{n^{1/2}}\right)$, and its bias is of order $O(\frac{1}{n})$.

\section{Proposed Method}
\label{Approach}
Throughout this paper, we consider learning functions taking the form $Y=f^\star(X)+\varepsilon$,
where $X$ and $\varepsilon$ are independent random variables drawn from a distribution $\mathcal{D}$. This presentation assumes the existence of a mechanism tying together $X$ and $Y$ through $f^\star$, up to independent noise factors.  
A typical learning approach is to set some hypothesis class $\mathcal{H}$, and attempt to solve the following problem:
\begin{equation*}
    \min_{h \in \mathcal{H}} \mathbb{E}_{X,\varepsilon\sim \mathcal{D}}[\ell(y,h(x))],
\end{equation*}
where $\ell$ is often the squared loss function in a regression setting, or the cross entropy loss in case of classification.

Here, following \citet{mooij2009regression}, we suggest using a loss function which penalizes hypotheses whose residual from $Y$ is not independent of the instance $X$. Concretely, we pose the following learning problem:
\begin{equation} \label{learning_problem}
    \min_{h \in \mathcal{H}} HSIC(X,Y-h(X);\F,\G),
\end{equation}
where we approximate the learning problem with empirical samples using $\widehat{HSIC}$ as shown in Eq. \eqref{eq:emphsic}. Unlike typical loss functions, this loss does not decompose as a sum of losses over each individual sample.
In Algorithm \ref{alg1} we present a general gradient-based method for learning with this loss.

\begin{algorithm}[tb]
   \caption{Learning with HSIC-loss}
   \label{alg1}
\begin{algorithmic}
   \STATE {\bfseries Input:} samples $\{(x_i,y_i)\}_{i=1}^n$, kernels $k$, $l$, a hypothesis $h_\theta$ parameterized by $\theta$, and a batch size $m>1$.
   \REPEAT
   \STATE Sample mini-batch $\{(x_i,y_i)\}_{i=1}^m$
   \STATE Compute the residuals $r_i^\theta=y_i-h_\theta(x_i)$
   \STATE Compute the kernel matrices $K_{i,j}=k(x_i,x_j)$, and $R_{i,j}^\theta=l(r_i^\theta,r_j^\theta)$
   \STATE Compute the HSIC-loss on the mini-batch: $\text{Loss}(\theta)=\textbf{tr}(KHR^\theta H)/(m-1)^2$ where  $H_{i,j}=\delta_{i,j}-\frac{1}{m}$
   \STATE Update: $\theta \leftarrow \theta - \alpha\cdot\nabla \text{Loss}(\theta)$
   \UNTIL convergence
   \STATE Compute the estimated source bias: \\ $b \leftarrow \frac{1}{n}\sum_{i=1}^n y_i- \frac{1}{n}\sum_{i=1}^n h_\theta(x_i)$
   \STATE {\bfseries Output:} A bias-adjusted hypothesis $h(x)=h_\theta(x)+b$
\end{algorithmic}
\end{algorithm}

As long as the kernel functions $k(\cdot,\cdot)$ and $l(\cdot,\cdot)$ are differentiable, taking the gradient of the HSIC-loss is simple with any automatic differentiation software  \citep{paszke2019pytorch,tensorflow2015-whitepaper}.
We note that $\text{HSIC}(X,Y-h(X);\F,\G)$ is exactly the same for any two functions $h_1(X),h_2(X)$ who differ only by a constant. This can be seen by the examining the role of the centering matrix $H$, or from the invariance of the covariance operator under constant shifts. Therefore, the predictor obtained from solving \eqref{learning_problem} is determined only up to a constant term. To determine the correct bias, one can infer it from the empirical mean of the observed $Y$ values. We note that it is possible to add a regularization term to the loss function. In our experiments we used standard regularization techniques such as L2 norm weight and early stopping, setting them by standard (source distribution only) cross-validation.   

\subsection{Understanding the HSIC-loss}
Here we provide two additional views on the HSIC-loss, motivating its use in cases which go beyond additive noise.

The first is based on the observation that, up to a constant, the residual $Y-h(X)$ is the gradient of the squared error with respect to $h(X)$. Intuitively, this means that by optimizing for the residual to be independent of $X$, we ask that the direction and magnitude in which we need to update $h(X)$ to improve the loss is the same regardless of $X$. Put differently, the gradient of $h(X)$ would be the same for every subset of $X$. This is also true for classification tasks: consider the outputs of a classification network as logits $o$ which are then transformed by Sigmoid or Softmax operations into a probability vector $h$. The gradient of the standard cross-entropy loss with respect to $o$ is exactly the residual $Y-h(X)$. Thus, even when not assuming additive noise, requiring that the residual would be independent of $X$ encourages learning a model for which the gradients of the loss have no information about the instances $X$. 

The second interpretation concerns the question of what does it mean for a model $h(X)$ to be optimal with respect to predicting $Y$ from $X$. One reasonable way to define optimality is when $h(X)$ captures all the available information that $X$ has about the label $Y$. That is, a classifier is optimal when:
\begin{equation} \label{optimal_classifier}
    Y \indep X \mid h(X).
\end{equation}
This is also related to the condition implied by recent work on Invariant Risk Minimization \citep{arjovsky2019invariant}. Optimizing for the condition in equation \ref{optimal_classifier} is difficult because of the conditioning on $h(X)$. We show in the supplemental that attaining the objective encouraged by the HSIC-loss, namely learning a function $h(X)$ such that $Y-h(X)\indep X$, implies the optimality condition \ref{optimal_classifier}.

\section{Theoretical Results}
\label{Theory}
We now prove several properties of the HSIC-loss, motivating its use as a loss function which emphasizes robustness against distribution shifts.
We consider models of the form given in Eq. \eqref{eq:model}  such that $\varepsilon$ has zero mean. 
Assume that $X \in \mathcal{X}$ and $Y \in \mathcal{Y}$, where $\mathcal{X},\mathcal{Y}$ are compact metric spa
ces. Denote by $\F$ and $\G$ reproducing kernel Hilbert spaces of functions from $\mathcal{X}$ and $\mathcal{Y}$ respectively, to $\mathbb{R}$ s.t. that $\|f\|_\F\le M_\F$ for all $f\in \F$ and $\|g\|_\G\le M_\G$ for all $g\in\G$. We will use $M_\G$ and $M_\F$ throughout this section.
Omitted proofs can be found in the supplement.
Denote by $\Tilde{\F}$ and $\Tilde{\G}$ the restriction of $\F$ and $\G$ to functions in the unit ball of the respective RKHS. Before we state the results, we give the following useful lemma:

\begin{lemma}\label{lem}
Suppose $\F$ and $\G$ are RKHSs over $\X$ and $\Y$, s.t. $\|s\|_\F \le M_\F$ for all $s\in \F$ and $\|t\|_\G \le M_\G$ for all $t \in \G$. Then the following holds:
{\small
\begin{align*}
    \sup_{s\in\F,t\in\G}\cov[s(X),t(Y)] = M_\F \cdot M_\G\sup_{s\in\Tilde{\F},t\in\Tilde{\G}}\cov[s(X),t(Y)].
\end{align*}}
\end{lemma}
\vspace{-5mm}

\subsection{Lower Bound}
We first relate HSIC-loss to standard notions of model performance: we show that under mild assumptions, the HSIC-loss is an upper bound to the variance of the residual $f^{\star}\left(X\right)-h\left(X\right)$. The additional assumption is that for all $h\in \mathcal{H}$, $f^\star -h$ is in the closure of $\F$, and that $\G$ contains the identity function from $\R$ to $\R$. This means that $M_\F$ acts as a measure of complexity of the true function $f^\star$ that we trying to learn. Note however this does not imply that $f^\star \in \mathcal{H}$, but rather this is an assumption on the kernel space used to calculate the HSIC term.

\begin{theorem} \label{lower bound}
Under the conditions specified above:
{\small
\begin{align*}
    \var(f^\star(X)-h(X)) \le M_\F \cdot M_\G \cdot \text{HSIC}(X,Y-h(X);\Tilde{\F},\Tilde{\G}).
\end{align*}}
\end{theorem}

Recalling the bias-variance decomposition:
\begin{align*}
    &\E\left[\left(Y-h\left(X\right)\right)^{2}\right]=\\
    &\var\left(f^{\star}\left(X\right)-h\left(X\right)\right)+\left(\E\left[f^{\star}\left(X\right)-h\left(X\right)\right]\right)^{2}+\E\left[\varepsilon^{2}\right],
\end{align*} 
we see that the HSIC-loss minimizes the variance part of the mean squared error (MSE). To minimize the entire MSE, the learned function should be adjusted by adding a constant which can be inferred from the empirical mean of $Y$.

\subsubsection{The Realizable Case}
If $h \in \mathcal{H}$ has HSIC-loss equal to zero, then up to a constant term, it is the correct function:

\begin{corollary} \label{right function}
Under the assumptions of Theorem \ref{lower bound}, we have the following: 
\begin{align*}
    \text{HSIC}\left(X,Y-h(X);\Tilde{\F},\Tilde{\G}\right)=0 \Rightarrow h\left(X\right)=f^{\star}\left(X\right)+c,
\end{align*} 
almost everywhere. 
\end{corollary}
\begin{proof}
From Theorem \ref{lower bound}, we have that 
\begin{align*}
    &\text{HSIC}\left(X,Y-h(X);\Tilde{\F},\Tilde{\G}\right)=0 \implies  \\ 
&\var\left(f^{\star}\left(X\right)-h\left(X\right)\right) =0 ,
\end{align*}
therefore $f^{\star}\left(X\right)-h\left(X\right)$ must be a constant up to a zero-probability set of $X$. 
\end{proof}

\subsection{Robustness Against Covariate Shift}
Due to its formulation as a supremum over a large set of functions, the HSIC-loss is an upper bound to a natural notion of robustness. This notion, which will be formalised below, captures the amount by which the performance of a model might change as a result of a covariate shift, where the performance is any that can be measured by some $\ell \in \G$ applied on the residuals $Y-h(X)$. In this subsection we denote the functions in $\G$ as $\ell$ instead of $t$, to emphasize that we now think of $\ell(r)$ as possible loss functions acting on the residuals. 

We consider two different ways of describing a target distribution which is different from the source.
The first is by specifying the density ratio between the target and source distributions. This is useful when the support of the distribution does not change but only the way it is distributed. A second type of covariate shift is due to restricting the support of the distribution to a certain subset. This can be described by an indicator function which states which parts of the source domain are included. The following shows how the HSIC-loss is an upper bound to the degradation in model performance in both covariate shift formulations.

We start with the latter case. For a subset $A\subset \X$ of positive measure, the quantity comparing a model's performance in terms of the loss $\ell$ on the source distribution and the same model's performance when the target distribution is restricted to $A$, is as follows:
\begin{equation*}
    \frac{1}{\mathbb{E}\left[1_A(x)\right]}\mathbb{E}\left[1_A(x)\ell\left(y-h\left(x\right)\right)\right]- \mathbb{E}\left[\ell\left(y-h\left(x\right)\right)\right]
\end{equation*}
For $\delta, c>0$ let $\mathcal{W}{c,\delta}$ denote the family of subsets $A$ with source probability at least $c>0$ s.t. there exists some $s\in \F$ which is $\delta$ close to $1_A$:
\begin{equation*}
    \mathcal{W}_{c,\delta}=\{ A \subset \X | \exists s \in \F \text{ s.t. } \|1_A-s\|_{\infty} \le \delta, \, \mathbb{E}\left[1_A(x)\right] \ge c \}.
\end{equation*}
All these subset can be approximately described by functions from $\F$. The complexity of such subsets is naturally controlled by $M_\F$. 
\begin{theorem}
    \label{subgroup}
    Let $\ell \in \G$ be a non-negative loss function, and let $\delta, c>0$:
    \begin{align*}
        &\sup_{A\in \mathcal{W}{c,\delta}} \frac{1}{\mathbb{E}\left[1_A(x)\right]}\mathbb{E}\left[1_A(x)\ell\left(y-h\left(x\right)\right)\right] \le \\  &\frac{M_\F M_\G HSIC(X,Y-h(X);\Tilde{\F},\Tilde{\G})}{c}  \\&+ \left(\frac{2\delta}{c}+1\right)\mathbb{E}\left[\ell\left(y-h\left(x\right)\right)\right].
    \end{align*}
\end{theorem}
Theorem \ref{subgroup} states that the degradation in performance due to restricting the support of the distribution to some subset is bounded by terms related to the size of the set and the ability to represent it by $\F$. Compare this to the following naive bound: 
\begin{equation*}
    \sup_{A\in \mathcal{W}{c,\delta}} \frac{\mathbb{E}\left[1_A(x)\ell\left(y-h\left(x\right)\right)\right]}{\mathbb{E}\left[1_A(x)\right]} \le \frac{\mathbb{E}\left[\ell\left(y-h\left(x\right)\right)\right]}{c}.
\end{equation*}
Failing to account how the loss is distributed across different subsets of $\X$, as done in the HSIC-loss, leads to poor generalization guarantees. Indeed, the naive bound will not be tight for the original function, i.e. $h=f^\star$, but the HSIC based bound will be tight whenever $\delta \ll c$. 

Returning to the first way of describing covariate shifts, we denote by $P_{\text{source}}(X)$ the density function of the distribution on $\mathcal{X}$ from which the training samples are drawn, and $P_{\text{target}}(X)$ the density of an unknown target distribution over $\mathcal{X}$. 

\begin{theorem} \label{robustness}
Let $\mathcal{Q}$ denote the set of density functions on $\mathcal{X}$ which are absolutely continuous w.r.t. $P_{\text{source}}(X)$, and their density ratio is in $\F$:
\begin{align*}
   \mathcal{Q} =  
   \bigg\{& P_\text{target}:\mathcal{X} \rightarrow \mathbb{R}_{\geq 0} \quad \text{s.t.  }  \E_{x\sim P_\text{target}}\left[1\right]=1,\\
   & \E_{x\sim P_\text{source}}\left[\frac{P_\text{target}(x)}{P_\text{source}(x)}\right]=1, \frac{P_\text{target}}{P_\text{source}} \in \F \bigg\}.  
\end{align*}
Then,
    {\small
    \begin{align*}
        &\sup_{\substack{P_\text{target} \in \mathcal{Q}\\\ell\in\G}} \E_{x\sim P_\text{target}} [\ell(Y-h(X))] - \E_{x\sim P_\text{source}} [\ell(Y-h(X))] \\
        &\le  M_\F \cdot M_\G \cdot \text{HSIC}(X,Y-h(X);\Tilde{\F},\Tilde{\G}),
    \end{align*}}
\end{theorem}
where HSIC is of course evaluated on the training distribution $P_\text{source}$.

Combining Theorem \ref{robustness} and the lower bound of Theorem \ref{lower bound}, we obtain the following result:
\begin{corollary}

\label{thm:together}
    Under the same assumptions of Theorems \ref{lower bound} and \ref{robustness}, further assume that the square function $x\mapsto x^2$, belongs to $\G$ or its closure. 
    Denote: $\delta_{\text{HSIC}}(h) = \text{HSIC}(X,Y-h(X);\Tilde{\F},\Tilde{\G})  $, $\text{MSE}_{P_\text{target}}(h) = \E_{P_\text{target}} [(Y-h(X))^2]$, \, $\text{bias}_\text{source}(h) = \E_{ P_\text{source}} [f^\star(x)-h(x)]$, and $\sigma^2 = \E[\varepsilon^2]$.
    Then:
    \begin{align*}
           &\sup_{P_\text{target} \in \mathcal{Q}} \text{MSE}_{{P}_\text{target}}(h) \\
           \le &2M_\F\cdot M_\G \cdot\delta_{\text{HSIC}}(h) +  \text{bias}_\text{source}(h)^2 + \sigma^2.
    \end{align*}
\end{corollary}

Theorem \ref{robustness} and Corollary \ref{thm:together} show that minimizing HSIC minimizes an upper bound on the worst case loss relative to a class of target distribution whose complexity is determined by the norm of the RKHS $M_\F$. Compared to a naive bound based on the infinity norm of the density ratio, this bound is much tighter when considering $f^\star$ for example, and by continuity for functions near it. Further discussion can be found in the supplementary material.

\subsection{Learnability: Uniform Convergence}
By formulating the HSIC learning problem as a learning problem over pairs of samples from $\mathcal{X} \times \mathcal{X}$ with specially constructed labels, we can reduce the question of HSIC learnability to a standard learning theory problem (\citet{mohri2018foundations}, Ch 3). We use this reduction to prove that it is possible to minimize the HSIC objective on hypothesis classes $\mathcal{H}$ with bounded Rademacher complexity $\mathcal{R}_n(\mathcal{H})$ using a finite sample.
\begin{theorem} \label{leanability_thm}
    Suppose the residuals' kernel $k$ is bounded in $[0,1]$ and satisfies the following condition: $k(r,r^\prime)=\iota(h(x)-h(x^\prime),y-y^\prime)$ where $\iota:\,\mathbb{R}\times\mathcal{Y}\to\mathbb{R}$ is s.t. $\iota(\cdot, y)$ is an $L_\iota-$Lipschitz function for all $y$. Let $C_1=\sup_{x,x^\prime}l(x,x^\prime)$, $C_2=\sup_{r,r^\prime}k(r,r^\prime)$. Then, with probability of at least $1-\delta$, the following holds for all $h\in\mathcal{H}$:
    {\small\begin{align*}
        &\left|\text{HSIC}\left(X,Y-h(X);\F,\G\right) - \widehat{\text{HSIC}}\left(\{(x_i,r_i)\}_{i=1}^n;\F,\G)\right)\right| \\
        \le &3C_1\left( 4L_\iota\mathcal{R}_n(\mathcal{H})+O\left(\sqrt{\frac{\ln(1/\delta)}{n}}\right)\right) + 3C_2C_1\sqrt{\frac{\ln(2/\delta)}{2n}}.
    \end{align*}}
\end{theorem}

\section{Related Work}
\label{Related}
As mentioned above, \citet{mooij2009regression} were the first to propose using the HSIC loss as a means to learn a regression model. However, their work focused exclusively on learning the functions corresponding to edges in a causal graph, and leveraging that to learn causal directions and then the structure of the graph itself. They have not applied the method to domain adaptation or to robust learning, nor did they analyze the qualities of this objective as a loss function.

The literature on robust learning is rapidly growing in size and we cannot hope to cover it all here. Especially relevant papers on robust learning for \emph{unsupervised} domain adaptation are \citet{namkoong2017variance,volpi2018generalizing,duchi2018learning}. \citet{volpi2018generalizing} propose an iterative process whereby the training set is augmented with adversarial examples that are close in some feature space, to obtain a perturbation of the distribution. \citet{namkoong2017variance} suggest minimizing the variance of the loss in addition to its empirical mean, and employ techniques from learning distributionally robust models.
Some recent papers have highlighted strong connections between causal inference and robust learning, see e.g. the works of \citet{heinze2017conditional} and \citet{rothenhausler2018anchor}.  By having some knowledge on the corresponding data generating graph of the problem, \citet{heinze2017conditional} propose minimizing the variance under properties that are presupposed to have no impact on the prediction. A more general means of using the causal graph to learn robust models is given by \citet{subbaswamy2018counterfactual,subbaswamy2019preventing}, who propose a novel \emph{graph surgery estimator} which specifically takes account of factors in the data which are known apriori to be vulnerable to changes in the distribution. These methods require detailed knowledge of the causal graph and are computationally heavy when the dimension of the problem grows. In \cite{rothenhausler2018anchor}, the authors propose using anchors, which are covariates that are known to be exogenous to the prediction problem, to obtain robustness against distribution shifts induced by the anchors. 
Of course, a large body of work exist on covariate shift learning when there is access to unlabeled test data (see, e.g., \citet{daume2006domain,  saenko2010adapting,gretton2009covariate,tzeng2017adversarial,volpi2018adversarial}), however we stress that we do not require such access.

\section{Experimental Results}
\label{Experiments}
To evaluate the performance of the HSIC loss function, we experiment with synthetic and real-world data. We focus on tasks of unsupervised transfer learning: we train on a one distribution, called the \textsc{source} distribution, and test on a different distribution, called the \textsc{target} distribution. We assume we have no samples from the target distribution during learning. 
\subsection{Synthetic Data}
\begin{figure*}[!t]
    {\includegraphics[scale=0.4]{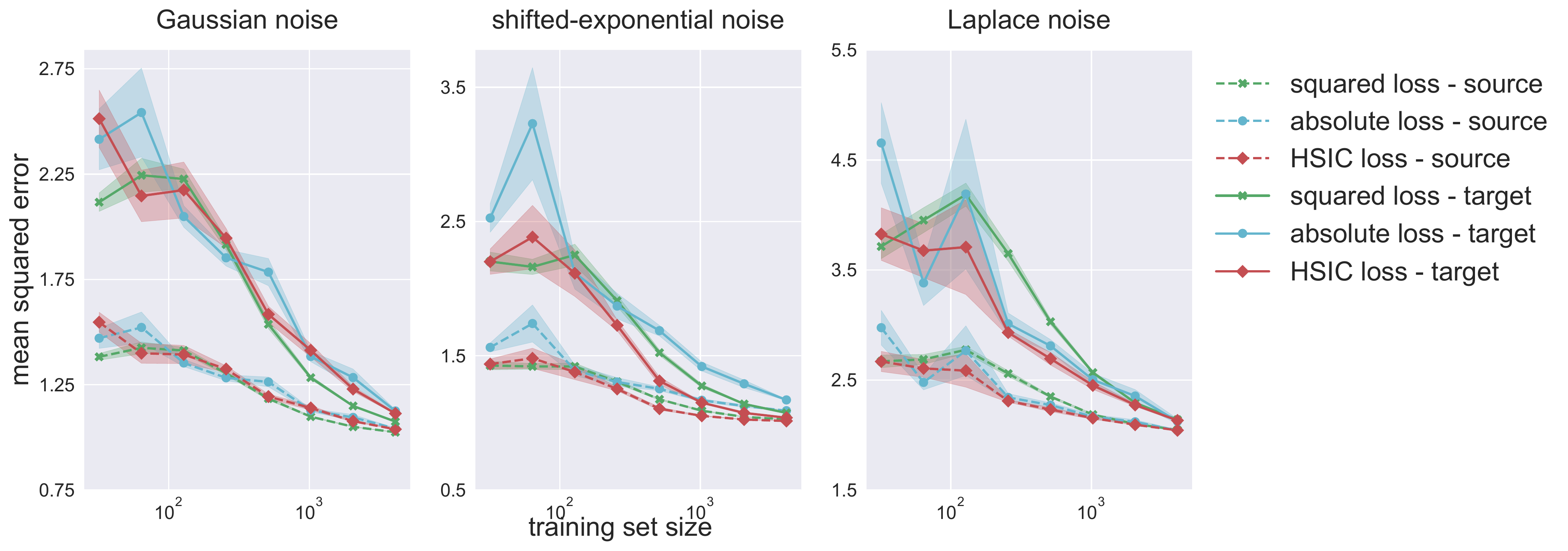}}
  \caption{Comparison of models trained with squared-loss, absolute-loss and HSIC-loss. Each point on the graph is the MSE averaged over 20 experiments, and the shaded area represents one standard error from the mean. Dashed lines are the MSE evaluated over the source distribution, solid lines are the MSE evaluated over the target distribution.}
  \label{linear_regression_results}
\end{figure*}
As a first evaluation of the HSIC-loss, we experiment with fitting a linear model. We focus on small sample sizes as those often lead to difficulties in covariate shift scenarios. The underlying model in the experiments is $y=\beta^\top x + \varepsilon$ where $\beta\in\mathbb{R}^{100}$ is drawn for each experiment from a Gaussian distribution with $\sigma=0.1$. In the training phase, $x$ is drawn from a uniform distribution over $[-1,1]^{100}$. We experimented with $\varepsilon$ drawn from one of three distributions: Gaussian, Laplacian, or a shifted exponential: $\varepsilon=1-e$ where $e$ is drawn from an exponential distribution $\exp(1)$. In any case, $\varepsilon$ is drawn independently from $x$.
In each experiment, we draw $n\in\{2^i\}_{i=5}^{13}$ training samples and train using either using squared-loss, absolute-loss, and HSIC-loss, all with an $l_2$ regularization term. 
The \textsc{source} test set is created in the same manner as the training set was created, while the \textsc{target} test set simulates a covariate shift scenario. This is done by changing the marginal distribution of $x$ from a uniform distribution to a Gaussian distribution over $\mathbb{R}^{100}$. In all cases the noise on the \textsc{source} and \textsc{target} is drawn from the same distribution. This process is repeated $20$ times for each $n$. When training the models with HSIC-loss, we used batch-size of 32, and optimized using Adam optimizer \citep{kingma2014adam}.
The kernels we chose were radial basis function kernels, with $\gamma=1$ for both covariates' and residuals' kernels. 

Figure \ref{linear_regression_results} presents the results of experiments with Gaussian, Laplacian, and shifted-exponential noise. With Gaussian noise, HSIC-loss performs similarly to squared-loss regression, and with Laplacian noise HSIC-loss performs similarly to absolute-loss regression, where squared-loss is the maximum-likelihood objective for Gaussian noise and absolute-loss is the maximum-likelihood objective for Laplacian noise.
In both cases it is reassuring to see that HSIC-loss is on par with the maximum-likelihood objectives. In all cases we see that HSIC-loss is better on the \textsc{target} distribution compared to objectives which are not the maximum likelihood objective. This is true especially in small sample sizes. We believe this reinforces our result in Theorem \ref{robustness} that the HSIC-loss is useful when we do not know in advance the loss or the exact target distribution. 
\subsection{Bike Sharing Dataset}
  In the bike sharing dataset by \citet{fanaee2014event} from the UCI repository, the task is to predict the number of hourly bike rentals based on the following covariates: temperature, feeling temperature, wind speed and humidity. Consisting of 17,379 samples, the data was collected over two years, and can be partitioned by year and season. This dataset has been used to examine domain adaptation tasks by \citet{subbaswamy2019preventing} and \citet{rothenhausler2018anchor}. We adopt their setup, where the \textsc{source} distribution used for training is three seasons of a year, and the \textsc{target} distribution used for testing is the forth season of the same year, and where the model of choice is linear. 
  We compare with least squares, anchor regression (AR) \citet{rothenhausler2018anchor} and Surgery by \citet{subbaswamy2019preventing}.
  
  We ran 100 experiments, each of them was done by randomly sub-sampling $80\%$ of the \textsc{source} set and $80\%$ of the \textsc{target} set, thus obtaining a standard error estimate of the mean. When training the models with HSIC-loss, we used batch-size of 32, and optimized the loss with Adam  \citep{kingma2014adam}, with learning rate drawn from a uniform distribution over $[0.0008,0.001]$. The kernels we chose were radial basis function kernels, with $\gamma=2$ for the covariates' kernel, and $\gamma=1$ for the residuals' kernel.
  
  We present the results in Table \ref{bike_sharing_variance}. Following the discussion in section \ref{Theory}, we report the \textit{variance} of the residuals in the test set. We can see that training with HSIC-loss results in better performances in 6 out of 8 times. In addition, unlike AR and Surgery, training with HSIC-loss does not require knowledge of the specific causal graph of the problem, nor does it require the training to be gathered from different sources as in AR.
\begin{table}[!hbt]
\centering
\caption{Variance results on the bike sharing dataset. Each row corresponds to a training set consisting of three season of that year, and the variance of $Y-h(X)$ on the \textsc{target} set consisting of the forth season is reported. In bold are the best results in each experiment, taking into account one standard error.}
\vspace{10pt}

{\small
\begin{tabular}{l@{\hskip 1mm}l@{\hskip 1mm}l@{\hskip 1mm}l@{\hskip 1mm}l@{\hskip 1mm}l} 
\hline
Test data     & OLS       & AR                      & Surgery   & HSIC                \\ 
\hline
(Y1) Season 1 & \textbf{15.4}$\pm$0.02 & \textbf{15.4}$\pm$0.02 & 15.5$\pm$0.03 & 16.0$\pm$0.04  \\
Season 2      &  23.1$\pm$0.03 & 23.1$\pm$0.03 & 23.7$\pm$0.04 & \textbf{22.9}$\pm$0.03  \\
Season 3     & 28.0$\pm$0.03 & 28.0$\pm$0.03 &  28.1$\pm$0.03 & \textbf{27.9}$\pm$0.03  \\
Season 4      & 23.7$\pm$0.03 & 23.7$\pm$0.03 & 25.6$\pm$0.04 & \textbf{23.6}$\pm$0.04  \\ 
\hline
(Y2) Season 1 & \textbf{29.8}$\pm$0.05 & \textbf{29.8}$\pm$0.05 & 30.7$\pm$0.06 & 30.7$\pm$0.07  \\
Season 2      & 39.0$\pm$0.05 & 39.1$\pm$0.05 & 39.2$\pm$0.06 & \textbf{38.9}$\pm$0.04          \\
Season 3      & 41.7$\pm$0.05 & 41.5$\pm$0.05 & 41.8$\pm$0.05 & \textbf{40.8}$\pm$0.05  \\
Season 4      & 38.7$\pm$0.04 & \textbf{38.6}$\pm$0.04 & 40.3$\pm$0.06 & \textbf{38.6}$\pm$0.05  \\
\hline
\end{tabular}}
\label{bike_sharing_variance}
\end{table}

\subsection{Rotating MNIST}
In this experiment we test the performance of models trained on the MNIST dataset by \citet{lecun1998gradient} as the \textsc{source} distribution, and digits which are rotated by an angle $\theta$ sampled from a uniform distribution over $[-45,45]$ as the \textsc{target} distribution. Samples of the test data are depicted in the supplementary material. The standard approach to obtain robustness against such perturbations is to augment the training data with images with similar transformations, as in \citet{796} for example. However, in practice it is not always possible to know in advance what kind of perturbations should be expected, and therefore it is valuable to develop methods for learning robust models even without such augmentations. We compared training with HSIC-loss to training with cross entropy loss, using three types of architectures. The first is a convolutional neural network (CNN):$\to$ \textit{input $\to$ conv(dim=32)$\to$ conv(dim=64)  $\to$
fully-connected(dim=524)$\to$ dropout(p=0.5) $\to$ fully-connected(dim=10)}. The second is a  multi-layered perceptron (MLP) with two hidden layers of size $256,524,1024$:
\textit{input $\to$ fully-connected(dim={256,524,1024}) $\to$ fully-connected(dim={256,524,1024})$\to$dropout(p=0.5) $\to$ fully-connected(dim=10)}. The third architecture was also an MLP, except there were four hidden layers instead of two.
Each experiment was repeated 20 times, and in every experiment the number of training steps (7 epochs) remained constant for a fair comparison. Each time the training set consisted of 10K randomly chosen samples. 
The kernels we chose were radial basis function kernels with $\gamma=1$ for the residuals, and $\gamma=22$ for the images, chosen according to the heuristics suggested by \citet{mooij2009regression}.
The results are depicted in Figure \ref{rotated_mnist_results}. We see that for all models, moving to the \textsc{target} distribution induces a large drop in accuracy. Yet for all architectures we see that using HSIC-loss gives better performance on the \textsc{target} set compared to using the standard cross-entropy loss.
\begin{figure}[]
    \centering
    \includegraphics[scale=0.28]{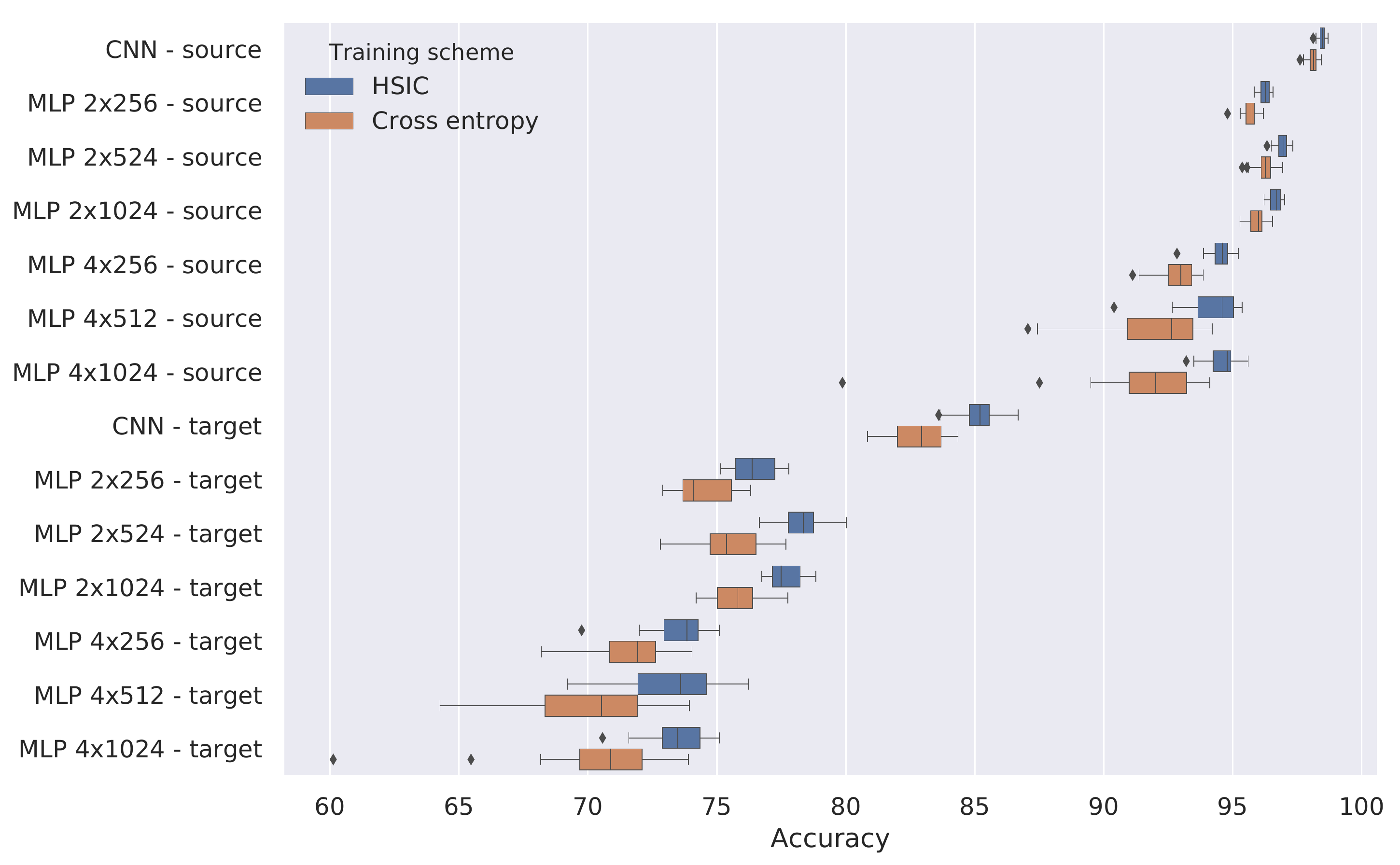}
      \caption{Accuracy on \textsc{source} and \textsc{target} test sets, with models trained with either cross entropy or HSIC-loss. Plotted are the median, 25th and 75th percentiles.}
    \label{rotated_mnist_results}
\end{figure}
\subsection{Cell Out of Sample Dataset}
In the last experiment, we test our approach on the cell out of sample dataset introduced by \citet{lu2019cells}. This dataset was collected for the purpose of measuring robustness against covariate shift. It consists of $64\times 64$ microscopy images of mouse cells stained with one of seven possible fluorescent proteins (highlighting distinct parts of the cell), and the task is to predict the type of the fluorescent protein used to stain the cell. Learning systems trained on microscopy data are known to suffer from changes in plates, wells and instruments \cite{caicedo2017data}. Attempting to simulate these conditions, the dataset contains four test sets, with increasing degrees of change in the covariates' distribution, as described in Table \ref{cell_description}, adopted from \citet{lu2019cells}.
Following \cite{lu2019cells}, we trained an 11-layer CNN, DeepLoc, used in \cite{kraus2017automated} for protein subcellular localization. We followed the pre-processing, data augmentation, architecture choice, and training procedures described there, with the exception of using HSIC-loss and different learning rate when using HSIC. When computing HSIC, the kernel width was set to 1 for both kernels. Training was done for 50 epochs on 80\% of the \textsc{source} dataset, and the final model was chosen according to the remaining 20\% used as a validation set. The optimization was done with Adam \cite{kingma2014adam}, with batch size of 128, and exponential decay of the learning rate was used when training with cross-entropy loss. \citet{lu2019cells} used data augmentation during training (random cropping and random flips) and test time (prediction is averaged over 5 crops taken from the corners and center image), as this is a common procedure to encourage robustness. We compared HSIC-based models to cross-entropy based models both with and without test time augmentation. We note that \cite{lu2019cells} examined several deep net models and DeepLoc (with cross-entropy training) had the best results.
\begin{table}
\caption{Description of the source and target distributions in the cell out of sample dataset}
{\small
\begin{tabular}{l@{\hskip 1mm}|p{54mm}@{\hskip 1mm}|l@{\hskip 1mm}l@{\hskip 1mm}l@{\hskip 1mm}l} 
\hline
Dataset     & Description       & Size \\ 
\hline
Source  & Images from 4 independent
plates for each class & 41,456
\\
Target1 & Held out data &10,364  \\
Target2 & Same plates,
but different wells & 17,021\\
Target3 & 2 independent plates for each class,
different days   &32,596
\\
Target4 & 1 plate for each class, 
different day and microscope & 30,772
\\
\hline
\end{tabular}}
\vspace{-10pt}

\label{cell_description}
\end{table}
\begin{table}
\caption{Class balanced accuracy on each of the four target distributions. The last row depicts the results of training with cross-entropy as reported in \cite{lu2019cells}. HSIC-aug and CE-aug refer to experiments done with test time augmentation.}
{\small
\begin{tabular}{l@{\hskip 2mm}l@{\hskip 2mm}l@{\hskip 2mm}l@{\hskip 2mm}l@{\hskip 2mm}l@{\hskip 2mm}l} 
\hline
Training loss     & Target1       & Target2 & Target3 & Target4 \\ 
\hline
HSIC  & 99.2 & 98.8 & 93.4 & 95.3
\\
CE  & 98.4 & 98.1 & 91.7 & 93.8
\\
\hline
HSIC-aug  & 99.2 & 98.9 & 93.4 & 95.4
\\
CE-aug-\cite{lu2019cells} &98.8 & 98.5 & 92.6 & 94.6
\\
\end{tabular}}
\vspace{-10pt}
\label{cell_results}
\end{table}

Table \ref{cell_results} depicts the results, showing a clear advantage of the HSIC-based model which is able to achieve new state-of-the-art results in the more difficult \textsc{target} distributions, while preserving the performance in \textsc{target} distributions closer to the \textsc{source}.

\section{Conclusion}
\label{Conclusion}
In this paper we propose learning models whose errors are independent of their inputs. This can be viewed as a non-parametric generalization of the way residuals are orthogonal to the instances in OLS regression. We prove that the HSIC-loss is learnable in terms of uniform convergence, and show that this loss naturally comes with a strong notion of robustness against changes in the input distribution when the change can be described in a bounded RKHS. The main theoretical limitations of our approach are the assumption that the changes are only in the marginal distribution of $X$, and that those changes are smooth in some sense. However, we show in our experiments that, in comparison to standard loss functions, the HSIC-loss produces models which perform just as well on the source distribution, and are significantly better on the target distribution, including state-of-the-art results on a challenging benchmark. An interesting future direction is to better understand the connection between this type of loss, originally proposed in the context of learning causal graph structure, and the idea of learning causal models that are expected to be robust against distributional changes \citep{meinshausen2018causality,arjovsky2019invariant}.

\section{Acknowledgments}
\label{Acknowledgments}
The authors would like to thank David Lopez-Paz for helpful discussions. This research was partially supported by the Israel Science Foundation (grant No. 1950/19).

\bibliography{main}

\begin{thebibliography}{31}
\providecommand{\natexlab}[1]{#1}
\providecommand{\url}[1]{\texttt{#1}}
\expandafter\ifx\csname urlstyle\endcsname\relax
  \providecommand{\doi}[1]{doi: #1}\else
  \providecommand{\doi}{doi: \begingroup \urlstyle{rm}\Url}\fi

\bibitem[Abadi et~al.(2015)Abadi, Agarwal, Barham, Brevdo, Chen, Citro,
  Corrado, Davis, Dean, Devin, Ghemawat, Goodfellow, Harp, Irving, Isard, Jia,
  Jozefowicz, Kaiser, Kudlur, Levenberg, Man\'{e}, Monga, Moore, Murray, Olah,
  Schuster, Shlens, Steiner, Sutskever, Talwar, Tucker, Vanhoucke, Vasudevan,
  Vi\'{e}gas, Vinyals, Warden, Wattenberg, Wicke, Yu, and
  Zheng]{tensorflow2015-whitepaper}
M.~Abadi, A.~Agarwal, P.~Barham, E.~Brevdo, Z.~Chen, C.~Citro, G.~S. Corrado,
  A.~Davis, J.~Dean, M.~Devin, S.~Ghemawat, I.~Goodfellow, A.~Harp, G.~Irving,
  M.~Isard, Y.~Jia, R.~Jozefowicz, L.~Kaiser, M.~Kudlur, J.~Levenberg,
  D.~Man\'{e}, R.~Monga, S.~Moore, D.~Murray, C.~Olah, M.~Schuster, J.~Shlens,
  B.~Steiner, I.~Sutskever, K.~Talwar, P.~Tucker, V.~Vanhoucke, V.~Vasudevan,
  F.~Vi\'{e}gas, O.~Vinyals, P.~Warden, M.~Wattenberg, M.~Wicke, Y.~Yu, and
  X.~Zheng.
\newblock {TensorFlow}: Large-scale machine learning on heterogeneous systems,
  2015.
\newblock URL \url{http://tensorflow.org/}.
\newblock Software available from tensorflow.org.

\bibitem[Arjovsky et~al.(2019)Arjovsky, Bottou, Gulrajani, and
  Lopez-Paz]{arjovsky2019invariant}
M.~Arjovsky, L.~Bottou, I.~Gulrajani, and D.~Lopez-Paz.
\newblock Invariant risk minimization.
\newblock \emph{arXiv preprint arXiv:1907.02893}, 2019.

\bibitem[Caicedo et~al.(2017)Caicedo, Cooper, Heigwer, Warchal, Qiu, Molnar,
  Vasilevich, Barry, Bansal, Kraus, et~al.]{caicedo2017data}
J.~C. Caicedo, S.~Cooper, F.~Heigwer, S.~Warchal, P.~Qiu, C.~Molnar, A.~S.
  Vasilevich, J.~D. Barry, H.~S. Bansal, O.~Kraus, et~al.
\newblock Data-analysis strategies for image-based cell profiling.
\newblock \emph{Nature methods}, 14\penalty0 (9):\penalty0 849, 2017.

\bibitem[Daume~III and Marcu(2006)]{daume2006domain}
H.~Daume~III and D.~Marcu.
\newblock Domain adaptation for statistical classifiers.
\newblock \emph{Journal of artificial Intelligence research}, 26:\penalty0
  101--126, 2006.

\bibitem[Duchi and Namkoong(2018)]{duchi2018learning}
J.~Duchi and H.~Namkoong.
\newblock Learning models with uniform performance via distributionally robust
  optimization.
\newblock \emph{arXiv preprint arXiv:1810.08750}, 2018.

\bibitem[Fanaee-T and Gama(2014)]{fanaee2014event}
H.~Fanaee-T and J.~Gama.
\newblock Event labeling combining ensemble detectors and background knowledge.
\newblock \emph{Progress in Artificial Intelligence}, 2\penalty0
  (2-3):\penalty0 113--127, 2014.

\bibitem[Gretton et~al.(2005{\natexlab{a}})Gretton, Bousquet, Smola, and
  Sch{\"o}lkopf]{gretton2005measuring}
A.~Gretton, O.~Bousquet, A.~Smola, and B.~Sch{\"o}lkopf.
\newblock Measuring statistical dependence with hilbert-schmidt norms.
\newblock In \emph{International conference on algorithmic learning theory},
  pages 63--77. Springer, 2005{\natexlab{a}}.

\bibitem[Gretton et~al.(2005{\natexlab{b}})Gretton, Herbrich, Smola, Bousquet,
  and Sch{\"o}lkopf]{gretton2005kernel}
A.~Gretton, R.~Herbrich, A.~Smola, O.~Bousquet, and B.~Sch{\"o}lkopf.
\newblock Kernel methods for measuring independence.
\newblock \emph{Journal of Machine Learning Research}, 6\penalty0
  (Dec):\penalty0 2075--2129, 2005{\natexlab{b}}.

\bibitem[Gretton et~al.(2008)Gretton, Fukumizu, Teo, Song, Sch{\"o}lkopf, and
  Smola]{gretton2008kernel}
A.~Gretton, K.~Fukumizu, C.~H. Teo, L.~Song, B.~Sch{\"o}lkopf, and A.~J. Smola.
\newblock A kernel statistical test of independence.
\newblock In \emph{Advances in neural information processing systems}, pages
  585--592, 2008.

\bibitem[Gretton et~al.(2009)Gretton, Smola, Huang, Schmittfull, Borgwardt, and
  Sch{\"o}lkopf]{gretton2009covariate}
A.~Gretton, A.~Smola, J.~Huang, M.~Schmittfull, K.~Borgwardt, and
  B.~Sch{\"o}lkopf.
\newblock Covariate shift by kernel mean matching.
\newblock \emph{Dataset shift in machine learning}, 3\penalty0 (4):\penalty0 5,
  2009.

\bibitem[Heinze-Deml and Meinshausen(2017)]{heinze2017conditional}
C.~Heinze-Deml and N.~Meinshausen.
\newblock Conditional variance penalties and domain shift robustness.
\newblock \emph{arXiv preprint arXiv:1710.11469}, 2017.

\bibitem[Kingma and Ba(2014)]{kingma2014adam}
D.~P. Kingma and J.~Ba.
\newblock Adam: A method for stochastic optimization.
\newblock \emph{arXiv preprint arXiv:1412.6980}, 2014.

\bibitem[Kraus et~al.(2017)Kraus, Grys, Ba, Chong, Frey, Boone, and
  Andrews]{kraus2017automated}
O.~Z. Kraus, B.~T. Grys, J.~Ba, Y.~Chong, B.~J. Frey, C.~Boone, and B.~J.
  Andrews.
\newblock Automated analysis of high-content microscopy data with deep
  learning.
\newblock \emph{Molecular systems biology}, 13\penalty0 (4), 2017.

\bibitem[LeCun et~al.(1998)LeCun, Bottou, Bengio, Haffner,
  et~al.]{lecun1998gradient}
Y.~LeCun, L.~Bottou, Y.~Bengio, P.~Haffner, et~al.
\newblock Gradient-based learning applied to document recognition.
\newblock \emph{Proceedings of the IEEE}, 86\penalty0 (11):\penalty0
  2278--2324, 1998.

\bibitem[Lu et~al.(2019)Lu, Lu, Schormann, Ghassemi, Andrews, and
  Moses]{lu2019cells}
A.~Lu, A.~Lu, W.~Schormann, M.~Ghassemi, D.~Andrews, and A.~Moses.
\newblock The cells out of sample (coos) dataset and benchmarks for measuring
  out-of-sample generalization of image classifiers.
\newblock In \emph{Advances in Neural Information Processing Systems}, pages
  1854--1862, 2019.

\bibitem[Meinshausen(2018)]{meinshausen2018causality}
N.~Meinshausen.
\newblock Causality from a distributional robustness point of view.
\newblock In \emph{2018 IEEE Data Science Workshop (DSW)}, pages 6--10. IEEE,
  2018.

\bibitem[Mohri et~al.(2018)Mohri, Rostamizadeh, and
  Talwalkar]{mohri2018foundations}
M.~Mohri, A.~Rostamizadeh, and A.~Talwalkar.
\newblock \emph{Foundations of machine learning}.
\newblock MIT press, 2018.

\bibitem[Mooij et~al.(2009)Mooij, Janzing, Peters, and
  Sch{\"o}lkopf]{mooij2009regression}
J.~Mooij, D.~Janzing, J.~Peters, and B.~Sch{\"o}lkopf.
\newblock Regression by dependence minimization and its application to causal
  inference in additive noise models.
\newblock In \emph{Proceedings of the 26th annual international conference on
  machine learning}, pages 745--752. ACM, 2009.

\bibitem[Namkoong and Duchi(2017)]{namkoong2017variance}
H.~Namkoong and J.~C. Duchi.
\newblock Variance-based regularization with convex objectives.
\newblock In \emph{Advances in Neural Information Processing Systems}, pages
  2971--2980, 2017.

\bibitem[Paszke et~al.(2019)Paszke, Gross, Massa, Lerer, Bradbury, Chanan,
  Killeen, Lin, Gimelshein, Antiga, et~al.]{paszke2019pytorch}
A.~Paszke, S.~Gross, F.~Massa, A.~Lerer, J.~Bradbury, G.~Chanan, T.~Killeen,
  Z.~Lin, N.~Gimelshein, L.~Antiga, et~al.
\newblock Pytorch: An imperative style, high-performance deep learning library.
\newblock In \emph{Advances in neural information processing systems}, pages
  8026--8037, 2019.

\bibitem[R{\'e}nyi(1959)]{renyi1959measures}
A.~R{\'e}nyi.
\newblock On measures of dependence.
\newblock \emph{Acta mathematica hungarica}, 10\penalty0 (3-4):\penalty0
  441--451, 1959.

\bibitem[Rosenfeld et~al.(2018)Rosenfeld, Zemel, and
  Tsotsos]{rosenfeld2018elephant}
A.~Rosenfeld, R.~Zemel, and J.~K. Tsotsos.
\newblock The elephant in the room.
\newblock \emph{arXiv preprint arXiv:1808.03305}, 2018.

\bibitem[Rothenh{\"a}usler et~al.(2018)Rothenh{\"a}usler, Meinshausen,
  B{\"u}hlmann, and Peters]{rothenhausler2018anchor}
D.~Rothenh{\"a}usler, N.~Meinshausen, P.~B{\"u}hlmann, and J.~Peters.
\newblock Anchor regression: heterogeneous data meets causality.
\newblock \emph{arXiv preprint arXiv:1801.06229}, 2018.

\bibitem[Saenko et~al.(2010)Saenko, Kulis, Fritz, and
  Darrell]{saenko2010adapting}
K.~Saenko, B.~Kulis, M.~Fritz, and T.~Darrell.
\newblock Adapting visual category models to new domains.
\newblock In \emph{European conference on computer vision}, pages 213--226.
  Springer, 2010.

\bibitem[Sch{\"o}lkopf et~al.(1996)Sch{\"o}lkopf, Burges, and Vapnik]{796}
B.~Sch{\"o}lkopf, C.~Burges, and V.~Vapnik.
\newblock Incorporating invariances in support vector learning machines.
\newblock In \emph{Artificial Neural Networks: ICANN 96, LNCS vol. 1112}, pages
  47--52, Berlin, Germany, July 1996. Max-Planck-Gesellschaft, Springer.
\newblock volume 1112 of Lecture Notes in Computer Science.

\bibitem[Sham~Kakade()]{Rademacher_Composition}
A.~T. Sham~Kakade.
\newblock Rademacher composition and linear prediction.
\newblock URL \url{https://ttic.uchicago.edu/~tewari/lectures/lecture17.pdf}.

\bibitem[Subbaswamy and Saria(2018)]{subbaswamy2018counterfactual}
A.~Subbaswamy and S.~Saria.
\newblock Counterfactual normalization: Proactively addressing dataset shift
  using causal mechanisms.
\newblock In \emph{34th Conference on Uncertainty in Artificial Intelligence
  2018, UAI 2018}, pages 947--957. Association For Uncertainty in Artificial
  Intelligence (AUAI), 2018.

\bibitem[Subbaswamy et~al.(2019)Subbaswamy, Schulam, and
  Saria]{subbaswamy2019preventing}
A.~Subbaswamy, P.~Schulam, and S.~Saria.
\newblock Preventing failures due to dataset shift: Learning predictive models
  that transport.
\newblock In \emph{The 22nd International Conference on Artificial Intelligence
  and Statistics}, pages 3118--3127, 2019.

\bibitem[Tzeng et~al.(2017)Tzeng, Hoffman, Saenko, and
  Darrell]{tzeng2017adversarial}
E.~Tzeng, J.~Hoffman, K.~Saenko, and T.~Darrell.
\newblock Adversarial discriminative domain adaptation.
\newblock In \emph{Proceedings of the IEEE Conference on Computer Vision and
  Pattern Recognition}, pages 7167--7176, 2017.

\bibitem[Volpi et~al.(2018{\natexlab{a}})Volpi, Morerio, Savarese, and
  Murino]{volpi2018adversarial}
R.~Volpi, P.~Morerio, S.~Savarese, and V.~Murino.
\newblock Adversarial feature augmentation for unsupervised domain adaptation.
\newblock In \emph{Proceedings of the IEEE Conference on Computer Vision and
  Pattern Recognition}, pages 5495--5504, 2018{\natexlab{a}}.

\bibitem[Volpi et~al.(2018{\natexlab{b}})Volpi, Namkoong, Sener, Duchi, Murino,
  and Savarese]{volpi2018generalizing}
R.~Volpi, H.~Namkoong, O.~Sener, J.~C. Duchi, V.~Murino, and S.~Savarese.
\newblock Generalizing to unseen domains via adversarial data augmentation.
\newblock In \emph{Advances in Neural Information Processing Systems}, pages
  5334--5344, 2018{\natexlab{b}}.

\end{thebibliography}
\bibliographystyle{abbrvnat}

\section{Appendix}
\label{Appendix}
\subsection{Batch-size effect}
Here we experiment with the effect of the batch-size on the performance of the HSIC-loss. The empirical HSIC value estimated on $m$ samples has bias of $O(1/m)$. This suggests that training with small batch-size might have a strong negative effect. However, in practice we have found that beyond the obvious computational cost of large batches (which scales as $m^2$), using larger mini-batches lead to worse performance in terms of accuracy. In the attached PDF we show that performance on the synthetic data setup using $m=64$ or 128 is considerably worse than $m=32$. We leave investigating the reasons for this interesting behavior, and the dynamics of gradients of HSIC in general, to future work.
\begin{figure}[]
    \centering
    \includegraphics[scale=0.6]{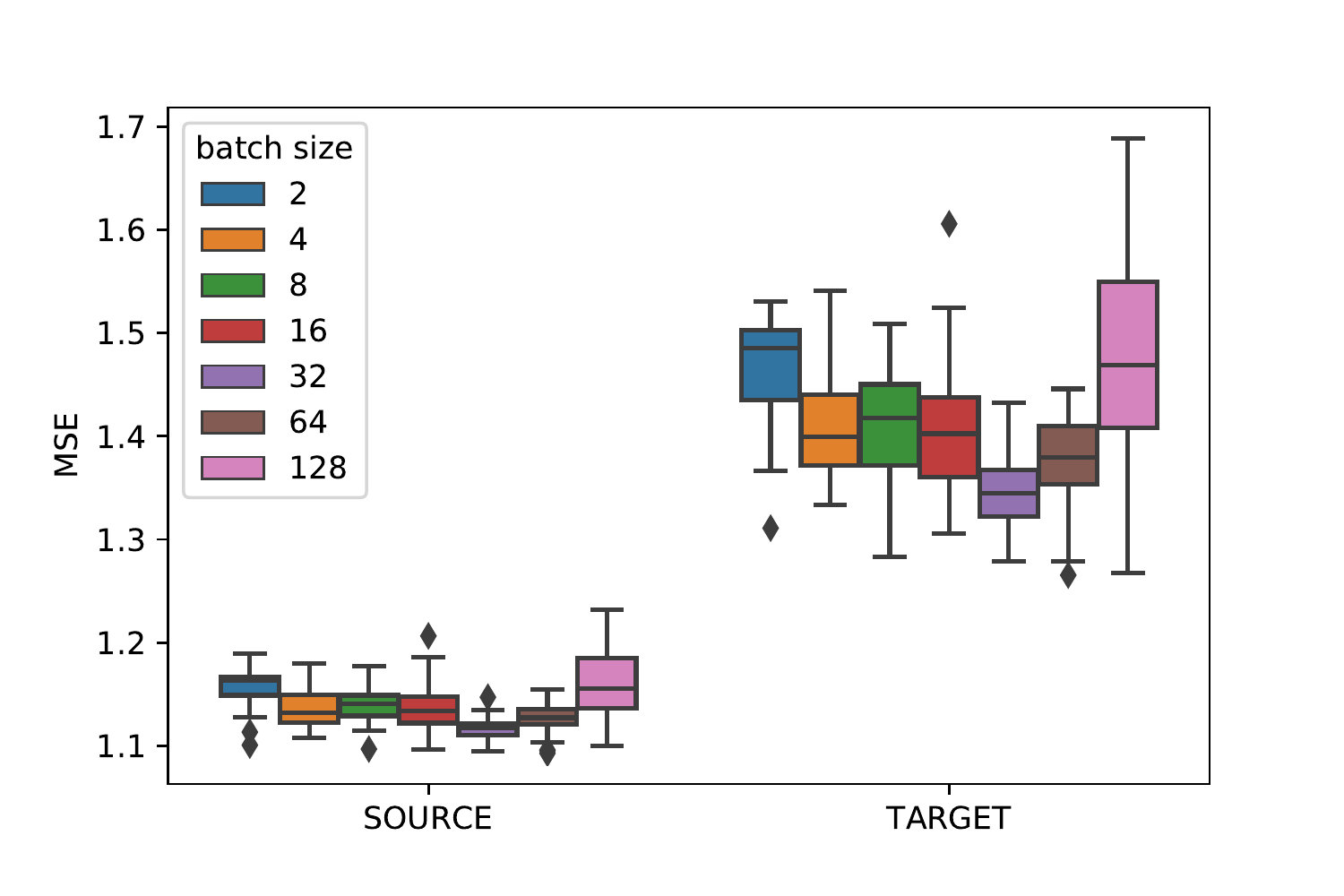}
      \caption{Effect of batch-size on performance on the linear regression experiment.}
    \label{batch_size_effect}
\end{figure}

\subsection{Theorems' implications}
In unsupervised covariate shift scenarios, there is always a tradeoff between how different do we allow the target distribution to be, versus the guarantees we can give for the target performance of a model trained on the source. In this work, we focus on specifying this tradeoff in terms of norms of RKHS spaces: the norm controls the complexity of the allowable target distribution set, and also multiplies components of the bounds. In the discussion of the theorems we argued that for $f^\star$ and functions near it, our bounds are much tighter compared to a bounds based on the infinity norm of the density ratio. Indeed, note that plugging $h=f^\star$ in Theorem \ref{robustness} we get: 
\begin{align*}
        &\sup_{\substack{P_\text{target} \in \mathcal{Q}\\\ell\in\G}} \E_{x\sim P_\text{target}} [\ell(Y-f^\star(X))] - \E_{x\sim P_\text{source}} [\ell(Y-f^\star(X))] \\
        &\le  0,
\end{align*}
which is a tight bound. When attempting to give a similar result using the infinity-norm of $\frac{P_\text{target}}{P_\text{source}} - 1$, the bounds become far from tight:
\begin{align*}
        &\sup_{\substack{P_\text{target} \in \mathcal{Q}\\\ell\in\G}} \E_{x\sim P_\text{target}} [\ell(Y-f^\star(X))] - \E_{x\sim P_\text{source}} [\ell(Y-f^\star(X))] \\
        &=\sup_{\substack{P_\text{target} \in \mathcal{Q}\\\ell\in\G}} \E_{x\sim P_\text{source}} \left[\left(\frac{P_\text{target}}{P_\text{source}}-1\right)\ell(Y-f^\star(X))\right] \\
        &\le C \cdot \sup_{\substack{P_\text{target} \in \mathcal{Q}\\\ell\in\G}} \E_{x\sim P_\text{source}} [\ell(Y-f^\star(X))],
\end{align*}
where $C$ bounds the infinity norm of $\frac{P_\text{target}}{P_\text{source}} - 1$. The fact that the HSIC bound is tighter justifies its use a loss function better suited for dealing with unsupervised covariate shift compared to simply minimizing a standard loss such as MSE.
This tighter bound comes at a cost: there are many cases where the infinity-norm of $\frac{P_\text{target}}{P_\text{source}} - 1$ is bounded but the Hilbert norm is unbounded, meaning that the density-ratio is not within the RKHS. Consider $p(k) = a/k^2, q(k) = b/(k+1)^2$ for $k=1,2,...$ , where $a$ and $b$ are the normalizing constants and $k$ is all positive integers. In this case $||q/p-1||_\infty$ is bounded by roughly $1.55$, but the $l_2$ norm diverges to infinity. This means that our set of distributions $\mathcal{Q}$ is a strictly smaller subset of the set of infinity-norm bounded ratios, while our bounds are correspondingly more informative.  

\subsection{Proofs}
Bellow we give the omitted proofs.
\subsubsection{Optimality proof}
We prove that $Y-h(X) \indep X$ implies that $Y\indep X \mid h(X)$. 
\begin{proof}
First, we note that this is equivalent to prove that $Y-h(X)\indep X \mid h(X)$, since the conditioning makes $h(X)$ just a constant. Now, 
\begin{equation*}
    \begin{split}
        &P\left(Y-h(X)\mid X,h(X)\right) \\
        &= P\left(Y-h(X)\mid X\right) \\
        &= P\left(Y-h(X)\right) \\
        &= P\left(Y-h(X)\mid h(X)\right),
    \end{split}
\end{equation*}
where the last equality is due to the fact that if the residual is independent of $X$, it is also independent of $h(X)$
\end{proof}
\subsubsection{Proof of lemma \ref{lem}}
\begin{proof}
We prove that $\sup_{s\in\F,t\in\G}\cov[s(X),t(Y)] \leq M_\F \cdot M_\G\sup_{s\in\Tilde{\F},t\in\Tilde{\G}}\cov[s(X),t(Y)]$. The other direction of the inequality is true by that same arguments. 

Let $\{s_i\}_{i=1}^\infty \subset \F$ and $\{t_i\}_{i=1}^\infty \subset \G$ be such that
\begin{align*}
    \lim_{n \to \infty}\cov[s_i(X),t_i(Y)]=\sup_{s\in\F,t\in\G}\cov[s(X),t(Y)].  
\end{align*}
Since for all $i$, $\frac{1}{M_\F}\cdot s_i \in \Tilde{\F}$ and likewise $\frac{1}{M_\G}\cdot t_i \in \Tilde{\G}$, we have that
\begin{align*}
    &\sup_{s\in\F,t\in\G}\cov[s(X),t(Y)] \\
    &= \lim_{n \to \infty}\cov[s_i(X),t_i(Y)] \\
    &= M_\F\cdot M_\G\lim_{n \to \infty}\cov[\frac{1}{M_\F}s_i(X),\frac{1}{M_\G}t_i(Y)] \\
    & \le M_\F\cdot M_\G \sup_{s\in\tilde{\F},t\in\tilde{\G}}\cov[s(X),t(Y)].
\end{align*}
\end{proof}
\subsubsection{Proof of theorem \ref{lower bound}}
\begin{proof}
Expanding the HSIC-loss:
\begin{align*}
&\text{HSIC}(X,Y-h(X);\Tilde{\F},\Tilde{\G})	\\\ge&\sup_{s\in\Tilde{\F},t\in\Tilde{\G}}\cov\left(s\left(X\right),t\left(Y-h\left(X\right)\right)\right) \\
=& \frac{1}{M_\F \cdot M_\G} \sup_{s\in\F,t\in\G}\cov\left(s\left(X\right),t\left(Y-h\left(X\right)\right)\right) \\
=& \frac{1}{M_\F \cdot M_\G}\sup_{s\in\F,t\in\G}\cov\left(s\left(X\right),t\left(f^{\star}\left(X\right)-h\left(X\right)+\varepsilon\right)\right),
\end{align*}
where the first inequality is due to Eq. \eqref{eq:hsiccoco}, and the first equality by Lemma \ref{lem}.
Now, by the assumption that $f^\star-h$ is in the closure of $\F$, there exist $s_{n}\in\F$ s.t. $ \left(s_{n}\right)_{n=1}^{\infty}$ converge to $f^{\star}-h$ under the infinity norm. 
Taking $t$ to be the identity function, we get that for all $n \in \mathbb{N}$:
\begin{align*}
    &\sup_{s\in\F}\cov\left(s\left(X\right),f^{\star}\left(X\right)-h\left(X\right)+\varepsilon\right) \\	&\ge\cov\left(s_{n}\left(X\right),t\left(f^{\star}\left(X\right)-h\left(X\right)+\varepsilon\right)\right),
\end{align*}
which implies that 
\begin{align*}
    &\sup_{s\in\F,t\in\G}\cov\left(s\left(X\right),t\left(f^{\star}\left(X\right)-h\left(X\right)+\varepsilon\right)\right)	\\
    \ge& \lim_{n\to \infty}\cov\left(s_{n}\left(X\right),f^{\star}\left(X\right)-h\left(X\right)+\varepsilon\right) \\
    =&\cov\left(f^{\star}\left(X\right)-h\left(X\right),f^{\star}\left(X\right)-h\left(X\right)+\varepsilon\right) \\
    =& \var\left(f^{\star}\left(X\right)-h\left(X\right)\right).
\end{align*}
\end{proof}

\subsubsection{Proof of theorem \ref{subgroup}}
\begin{proof}
First, we note that 
\begin{equation*}
    \begin{split}
        &\left|\mathbb{E}\left[s_{A}\left(x\right)\ell\left(y-h(x)\right)\right]-\mathbb{E}\left[1_{A}\left(x\right)\ell\left(y-h(x)\right)\right]\right| \\ &\le	\left|\mathbb{E}\left[\left(s_{A}\left(x\right)-1_{A}\left(x\right)\right)\ell\left(y-h(x)\right)\right]\right| \\&
        \le	\mathbb{E}\left[\left|s_{A}\left(x\right)-1_{A}\left(x\right)\right|\ell\left(y-h(x)\right)\right] \\&
        \le	\delta\mathbb{E}\left[\ell\left(y-h(x)\right)\right].
    \end{split}
\end{equation*}
And therefore
\begin{equation*}
    \begin{split}
        &\mathbb{E}\left[1_{A}\left(x\right)\ell\left(y-h(x)\right)\right] \\
        &\le \mathbb{E}\left[s_{A}\left(x\right)\ell\left(y-h(x)\right)\right]+\delta\mathbb{E}\left[\ell\left(y-h(x)\right)\right].
    \end{split}
\end{equation*}
Now, by definition, 
\begin{equation*}
    \begin{split}
        &\mathbb{E}\left[s_{A}\left(x\right)\ell\left(y-h(x)\right)\right]\\	&\le \mathbb{E}\left[s_{A}\left(x\right)\right]\mathbb{E}\left[\ell\left(y-h(x)\right)\right]\\&+M_{\mathcal{F}}M_{\mathcal{G}}HSIC\left(x,y-h(x);\tilde{\mathcal{F}},\tilde{\mathcal{G}}\right).
    \end{split}
\end{equation*}
Similar to before, 
\begin{equation*}
    \begin{split}
        &\left|\mathbb{E}\left[s_{A}\left(x\right)\right]\mathbb{E}\left[\ell\left(y-h(x)\right)\right]-\mathbb{E}\left[1_{A}\left(x\right)\right]\mathbb{E}\left[\ell\left(y-h(x)\right)\right]\right|\\
        &\le\delta\mathbb{E}\left[\ell\left(y-h(x)\right)\right],
    \end{split}
\end{equation*}
and therefore,
\begin{equation*}
    \begin{split}
        &\mathbb{E}\left[s_{A}\left(x\right)\right]\mathbb{E}\left[\ell\left(y-h(x)\right)\right]\\
        &\le\mathbb{E}\left[1_{A}\left(x\right)\right]\mathbb{E}\left[\ell\left(y-h(x)\right)\right]+\delta\mathbb{E}\left[\ell\left(y-h(x)\right)\right].
    \end{split}
\end{equation*}
Combining all the above:
\begin{equation*}
    \begin{split}
        &\frac{\mathbb{E}\left[1_{A}\left(x\right)\ell\left(y-h(x)\right)\right]}{\mathbb{E}\left[1_{A}\right]} \\
        &\le \frac{\mathbb{E}\left[s_{A}\left(x\right)\ell\left(y-h(x)\right)\right]+\delta\mathbb{E}\left[\ell\left(y-h(x)\right)\right]}{\mathbb{E}\left[1_{A}\right]} \\
        &\le \frac{\mathbb{E}\left[s_{A}\left(x\right)\right]\mathbb{E}\left[\ell\left(y-h(x)\right)\right]}{\mathbb{E}\left[1_{A}\right]} \\
        &+\frac{M_{\mathcal{F}}M_{\mathcal{G}}HSIC\left(x,y-h(x);\tilde{\mathcal{F}},\tilde{\mathcal{G}}\right)+\delta\mathbb{E}\left[\ell\left(y-h(x)\right)\right]}{\mathbb{E}\left[1_{A}\right]}\\
        &\le \frac{\mathbb{E}\left[1_{A}\left(x\right)\right]\mathbb{E}\left[\ell\left(y-h(x)\right)\right]+\delta\mathbb{E}\left[\ell\left(y-h(x)\right)\right]}{\mathbb{E}\left[1_{A}\right]} \\
        &+ \frac{M_{\mathcal{F}}M_{\mathcal{G}}HSIC\left(x,y-h(x);\tilde{\mathcal{F}},\tilde{\mathcal{G}}\right)+\delta\mathbb{E}\left[\ell\left(y-h(x)\right)\right]}{\mathbb{E}\left[1_{A}\right]} \\
        &= \frac{M_{\mathcal{F}}M_{\mathcal{G}}HSIC\left(x,y-h(x);\tilde{\mathcal{F}},\tilde{\mathcal{G}}\right)+2\delta\mathbb{E}\left[\ell\left(y-h(x)\right)\right]}{\mathbb{E}\left[1_{A}\right]} \\
        &+\mathbb{E}\left[\ell\left(y-h(x)\right)\right] \\
        &\le \frac{M_{\mathcal{F}}M_{\mathcal{G}}HSIC\left(x,y-h(x);\tilde{\mathcal{F}},\tilde{\mathcal{G}}\right)+2\delta\mathbb{E}\left[\ell\left(y-h(x)\right)\right]}{c} \\
        &+\mathbb{E}\left[\ell\left(y-h(x)\right)\right] \\
        &=\frac{M_{\mathcal{F}}M_{\mathcal{G}}HSIC\left(x,y-h(x);\tilde{\mathcal{F}},\tilde{\mathcal{G}}\right)}{c} \\
        &+\left(\frac{2\delta}{c}+1\right)\mathbb{E}\left[\ell\left(y-h(x)\right)\right]
    \end{split}
\end{equation*}
\end{proof}

\subsubsection{Proof of theorem \ref{robustness}}
\begin{proof}
We have:
\begin{equation*}
    \begin{split}
     &\text{HSIC}(X,Y-h(X);\Tilde{\F},\Tilde{\G}) \\
     \ge& \sup_{s\in \Tilde{\F}, \ell\in \Tilde{\G}}\bigg(\E_{x\sim\mathcal{P}}[s(X)\ell(Y-h(X))]\\
     &-E_{x\sim\mathcal{P}_\text{source}}[s(X)]E_{x\sim\mathcal{P}_\text{source}}[\ell(Y-h(X))]\bigg)\\
     =& \frac{1}{M_\F \cdot M_\G} \sup_{s\in \F, \ell\in \G}\bigg(\E_{x\sim\mathcal{P}_\text{source}}[s(X)\ell(Y-h(X))] \\
     &-\E_{x\sim\mathcal{P}_\text{source}}[s(X)]\E_{x\sim\mathcal{P}_\text{source}}[\ell(Y-h(X))] \bigg)\\
    \ge&\frac{1}{M_\F \cdot M_\G}\sup_{s\in \mathcal{S}, \ell\in \G}\bigg(\E_{x\sim\mathcal{P}_\text{source}}[s(X)\ell(Y-h(X))]\\
    &-\E_{x\sim\mathcal{P}_\text{source}}[s(X)]\E_{x\sim\mathcal{P}_\text{source}}[\ell(Y-h(X))] \bigg)\\
    =&\frac{1}{M_\F \cdot M_\G} \sup_{\mathcal{P}_\text{target} \in \mathcal{Q}}
    \sup_{\ell\in \G}\bigg(\E_{x\sim\mathcal{P}_\text{target}}[\ell(Y-h(X))]\\
    &-\E_{x\sim\mathcal{P}_\text{source}}[\ell(Y-h(X))]\bigg).
    \end{split}
\end{equation*}

The first equality is an immediate result of \eqref{eq:hsiccoco} and the identity $\cov(A,B) = \E[A B]-\E[A]\E[B]$. The second inequality is by the restriction from $\F$ to $\mathcal{S}$. The final equality is by the assumption that for all $\mathcal{P}_\text{target} \in \mathcal{Q}$, $\E_{x\sim \mathcal{P}}\left[\frac{P_\text{target}}{P_\text{source}}(x)\right]=1$, i.e. that $\mathcal{P}_\text{target}$  is absolutely continuous with respect to $\mathcal{P}_\text{source}$.
\end{proof}

\subsubsection{Proof of theorem \ref{thm:together}}
\begin{proof}
By the lower bound of Theorem \ref{lower bound}, we get $\var(f^\star(X)-h(X)) \le M_\F\cdot M_\G \cdot \delta_{\text{HSIC}}$. By Theorem \ref{robustness} and the assumption we get that:
{\small
\begin{align*}
    &M_\F\cdot M_\G \cdot \delta_{\text{HSIC}}(h) \\
    \ge& \sup_{\mathcal{P}_\text{target} \in \mathcal{Q}} \E_{x\sim \mathcal{P}_\text{target}} [(Y-h(X))^2] - \E_{x\sim \mathcal{P}_\text{source}} [(Y-h(X))^2] \\
    =& \sup_{\mathcal{P}_\text{target} \in \mathcal{Q}} \bigg(\E_{x\sim \mathcal{P}_\text{target}} [(Y-h(X))^2] \\
    &- \E_{x\sim \mathcal{P}_\text{source}} [(f^\star(X)+\varepsilon-h(X))^2]\bigg) \\
    =& \sup_{\mathcal{P}_\text{target} \in \mathcal{Q}} \bigg(\E_{x\sim \mathcal{P}_\text{target}} [(Y-h(X))^2]- \var[f^\star(X)-h(X)] \\
    & - (\E_{x\sim \mathcal{P}_\text{source}} [f^\star(X)-h(X)])^2-\E[\varepsilon^2]\bigg) \\
    =& \sup_{\mathcal{P}_\text{target} \in \mathcal{Q}}\bigg(\textit{MSE}_{P_\text{target}}(h) - \var[f^\star(X)-h(X)] \\
    &- {\textit{bias}_\text{source}(h)}^2 - \sigma^2\bigg).
\end{align*}}
Together, these inequalities give the result.
\end{proof}

\subsubsection{Proof of theorem \ref{leanability_thm}}
Bellow we give the proof of Theorem \ref{leanability_thm}, showing uniform convergence for the HSIC-loss. The basic idea is to reduce the problem to a standard learning problem of the form $\sup_{h\in\mathcal{H}}\left|z(h)-\frac{1}{n}\sum_{i=1}^n z_i(h)\right|$, where $z$ is some statistic, and $z_i(h)$ are i.i.d. samples. To do so, we follow \citet{gretton2005measuring}, decomposing the HSIC-loss to three terms, and after some algebraic manipulations, and under some assumptions about the form of the kernel, we show that there is one term which we need to bound:
$$\sup_{h\in\mathcal{H}}\left|\mathbb{E}_{r,r^{\prime}}\left[k\left(r,r^{\prime}\right)\right]-\frac{1}{\left(n\right)_{2}}\sum_{i_{2}^{n}}k\left(r_{i_{1}},r_{i_{2}}\right)\right|.$$
We then show how we can treat this as a learning problem over pairs of instances, where the objective is to predict the difference in $y$, allowing us to use standard tools and concentration bounds.

Recall that the empirical estimation problem we pose is
{
\begin{align}\label{eq:l_problem}
    \min_\theta\widehat{\text{HSIC}}\{(x_i,r^\theta_i)\}_{i=1}^n;\F,\G) = \min_\theta\frac{1}{(m-1)^2} \textbf{tr} R^\theta HKH,
\end{align}}
where $R^\theta _{i,j}=k(r_i^\theta,r_j^\theta),\, r_l^\theta=y_l-h_\theta (x_l),$ and $K_{i,j}=l(x_i,x_j)$ and by the cyclic property of the trace we switched the positions of $R^\theta$ and $K$.
\begin{lemma}
Let $C_1=\sup_{x,x^\prime}l(x,x^\prime)$, $C_2=\sup_{r,r^\prime}k(r,r^\prime)$. Then the following holds:
{\small
\begin{align*}
    &\sup_{h \in \mathcal{H}} \left|\text{HSIC}(X,Y-h(X);\F,\G) - \widehat{\text{HSIC}}\left(\{(x_i,r_i)\}_{i=1}^n;\F,\G\right)\right|\\
    \le &3C_1 \cdot \sup_{h\in \mathcal{H}}  \left|\E_{r,r^\prime}[k(r,r^\prime)] - \frac{1}{(n)_2} \sum_{i_2^n}k(r_{i_1},r_{i_2})\right| \\
    &+ 3C_2 \cdot \left|\E_{x,x^\prime} [l(x,x^\prime] - \frac{1}{(n)_2}\sum_{i_2^n}l(x_{i_1},x_{i_2})\right|
\end{align*}
}\label{learnability lemma}
\end{lemma}
\begin{proof}
Following \cite{gretton2005measuring}, HSIC can be decomposed into a three part sum:
\begin{equation*}
    \begin{split}
    &HSIC\left(X,Y-h\left(X\right);\mathcal{F},\mathcal{G}\right)\\
    =&\mathbb{E}_{x,x^{\prime},r,r^{\prime}}\left[k\left(r,r^{\prime}\right)l\left(x,x^{\prime}\right)\right]\\
    &-2\mathbb{E}_{x,r}\left[\mathbb{E}_{x^{\prime}}\left[l\left(x,x^{\prime}\right)\right]\mathbb{E}_{r^{\prime}}\left[k\left(r,r^{\prime}\right)\right]\right]\\
    &+\mathbb{E}_{r,r^{\prime}}\left[k\left(r,r^{\prime}\right)\right]\mathbb{E}_{x,x^{\prime}}\left[l\left(x,x^{\prime}\right)\right].
    \end{split}
\end{equation*}
And likewise, the empirical HSIC can be decomposed as follows:
\begin{equation*}
    \begin{split}
        &\widehat{HSIC}\left(\left\{ \left(x_{i},r_{i}\right)\right\} _{i=1}^{n};\mathcal{F},\mathcal{G}\right)\\
        =&\frac{1}{\left(n\right)_{2}}\sum_{i_{2}^{n}}k\left(r_{i_{1}},r_{i_{2}}\right)l\left(x_{i_{1}},x_{i_{2}}\right)\\
        &-\frac{2}{\left(n\right)_{3}}\sum_{i_{3}^{n}}k\left(r_{i_{1}},r_{i_{2}}\right)l\left(x_{i_{2}},x_{i_{3}}\right)\\
        &+\frac{1}{\left(n\right)_{4}}\sum_{i_{4}^{n}}k\left(r_{i_{1}},r_{i_{2}}\right)l\left(x_{i_{3}},x_{i_{4}}\right)+O\left(\frac{1}{n}\right).
    \end{split}
\end{equation*}
where $\left(n\right)_{m}=\frac{n!}{\left(n-m\right)!}$ and $i_r^n$ is the set of all $r-$tuples drawn without replacement from $[n]$. From this we can see that it is enough to bound the following three terms:
\begin{align}\label{term_1}
\sup_{h\in\mathcal{H}}& \left|\vphantom{\sum_{i_{2}^{n}}}\mathbb{E}_{x,x^{\prime},r,r^{\prime}}\left[k\left(r,r^{\prime}\right)l\left(x,x^{\prime}\right)\right]\right. \\
&\left. -\frac{1}{\left(n\right)_{2}}\sum_{i_{2}^{n}}k\left(r_{i_{1}},r_{i_{2}}\right)l\left(x_{i_{1}},x_{i_{2}}\right)\right|, 
\end{align}
\begin{equation}
\begin{split} \label{term_2}
\sup_{h\in\mathcal{H}}&\left|\vphantom{\sum_{i_{2}^{n}}}\mathbb{E}_{x,r}\left[\mathbb{E}_{x^{\prime}}\left[l\left(x,x^{\prime}\right)\right]\mathbb{E}_{r^{\prime}}\left[k\left(r,r^{\prime}\right)\right]\right] \right.\\ 
& \left.-\frac{1}{\left(n\right)_{3}}\sum_{i_{3}^{n}}k\left(r_{i_{1}},r_{i_{2}}\right)l\left(x_{i_{2}},x_{i_{3}}\right)\right|,
\end{split}
\end{equation}
\begin{equation} \label{term_3}
\begin{split}
\sup_{h\in\mathcal{H}}&\left|\vphantom{\sum_{i_{2}^{n}}}\mathbb{E}_{r,r^{\prime}}\left[k\left(r,r^{\prime}\right)\right]\mathbb{E}_{x,x^{\prime}}\left[l\left(x,x^{\prime}\right)\right]\right.\\
&\left.-\frac{1}{\left(n\right)_{4}}\sum_{i_{4}^{n}}k\left(r_{i_{1}},r_{i_{2}}\right)l\left(x_{i_{3}},x_{i_{4}}\right)\right|.
\end{split}
\end{equation}
Using simple algebra, one can obtain the following bound for \eqref{term_1}:
{\tiny
\begin{equation*}
    \begin{split}
    &\sup_{h\in\mathcal{H}}\left|\mathbb{E}_{x,x^{\prime},r,r^{\prime}}\left[k\left(r,r^{\prime}\right)l\left(x,x^{\prime}\right)\right]-\frac{1}{\left(n\right)_{2}}\sum_{i_{2}^{n}}k\left(r_{i_{1}},r_{i_{2}}\right)l\left(x_{i_{1}},x_{i_{2}}\right)\right| \\
	=&\sup_{h\in\mathcal{H}}\left|\mathbb{E}_{x,x^{\prime},r,r^{\prime}}\left[k\left(r,r^{\prime}\right)l\left(x,x^{\prime}\right)\right]-\frac{1}{\left(n\right)_{2}}\sum_{i_{2}^{n}}k\left(r_{i_{1}},r_{i_{2}}\right)l\left(x_{i_{1}},x_{i_{2}}\right) \right. \\
	&\left. +\frac{1}{\left(n\right)_{2}}\sum_{i_{2}^{n}}k\left(r_{i_{1}},r_{i_{2}}\right)\mathbb{E}\left[l\left(x,x^{\prime}\right)\right]-\frac{1}{\left(n\right)_{2}}\sum_{i_{2}^{n}}k\left(r_{i_{1}},r_{i_{2}}\right)\mathbb{E}\left[l\left(x,x^{\prime}\right)\right]\right| \\
	=&\sup_{h\in\mathcal{H}}\left|\mathbb{E}_{x,x^{\prime},r,r^{\prime}}\left[\left(k\left(r,r^{\prime}\right)-\frac{1}{\left(n\right)_{2}}\sum_{i_{2}^{n}}k\left(r_{i_{1}},r_{i_{2}}\right)\right)l\left(x,x^{\prime}\right)\right] \right.\\
	&\left. +\frac{1}{\left(n\right)_{2}}\sum_{i_{2}^{n}}k\left(r_{i_{1}},r_{i_{2}}\right)\left(\mathbb{E}\left[l\left(x,x^{\prime}\right)\right]-l\left(x_{i_{1}},x_{i_{2}}\right)\right)\right|\\
	\le&\sup_{h\in\mathcal{H}}\left|\mathbb{E}_{x,x^{\prime},r,r^{\prime}}\left[\left(k\left(r,r^{\prime}\right)-\frac{1}{\left(n\right)_{2}}\sum_{i_{2}^{n}}k\left(r_{i_{1}},r_{i_{2}}\right)\right)l\left(x,x^{\prime}\right)\right]\right|\\
	&+\sup_{h\in\mathcal{H}}\left|\frac{1}{\left(n\right)_{2}}\sum_{i_{2}^{n}}k\left(r_{i_{1}},r_{i_{2}}\right)\left(\mathbb{E}\left[l\left(x,x^{\prime}\right)\right]-l\left(x_{i_{1}},x_{i_{2}}\right)\right)\right| \\
	\le&C_1\sup_{h\in\mathcal{H}}\left|\mathbb{E}_{x,x^{\prime},r,r^{\prime}}\left[k\left(r,r^{\prime}\right)-\frac{1}{\left(n\right)_{2}}\sum_{i_{2}^{n}}k\left(r_{i_{1}},r_{i_{2}}\right)\right]\right|\\
	&+C_2\sup_{h\in\mathcal{H}}\left|\mathbb{E}\left[l\left(x,x^{\prime}\right)\right]-\frac{1}{\left(n\right)_{2}}\sum_{i_{2}^{n}}l\left(x_{i_{1}},x_{i_{2}}\right)\right| \\
	=&C_1\sup_{h\in\mathcal{H}}\left|\mathbb{E}_{r,r^{\prime}}\left[k\left(r,r^{\prime}\right)\right]-\frac{1}{\left(n\right)_{2}}\sum_{i_{2}^{n}}k\left(r_{i_{1}},r_{i_{2}}\right)\right|\\
	&+C_2\left|\mathbb{E}\left[l\left(x,x^{\prime}\right)\right]-\frac{1}{\left(n\right)_{2}}\sum_{i_{2}^{n}}l\left(x_{i_{1}},x_{i_{2}}\right)\right|.
	\end{split}
\end{equation*}
}
Where the first inequality follows from properties of $\sup$, the second inequality follows from the definitions of $C_1$ and $C_2$, and the last equality follows from the fact that the second term does not depend on $h$.
Similarly, \eqref{term_2} can be bounded as follows:
{\tiny
    \begin{equation*}
        \begin{split}
        &\sup_{h\in\mathcal{H}}\left|\vphantom{\sum_{i_{2}^{n}}}\mathbb{E}_{x,r}\left[\mathbb{E}_{x^{\prime}}\left[l\left(x,x^{\prime}\right)\right]\mathbb{E}_{r^{\prime}}\left[k\left(r,r^{\prime}\right)\right]\right]\right.\\
        &\left.-\frac{1}{\left(n\right)_{3}}\sum_{i_{3}^{n}}k\left(r_{i_{1}},r_{i_{2}}\right)l\left(x_{i_{2}},x_{i_{3}}\right)\right| \\
        = &\sup_{h\in\mathcal{H}}\left|\vphantom{\sum_{i_{2}^{n}}}\mathbb{E}_{x,r}\left[\mathbb{E}_{x^{\prime}}\left[l\left(x,x^{\prime}\right)\right]\mathbb{E}_{r^{\prime}}\left[k\left(r,r^{\prime}\right)\right]\right]\right.\\
        &\left.-\frac{1}{\left(n\right)_{3}}\sum_{i_{3}^{n}}k\left(r_{i_{1}},r_{i_{2}}\right)l\left(x_{i_{2}},x_{i_{3}}\right)\right.\\
        &+\left.{} \frac{1}{\left(n\right)_{3}}\sum_{i_{3}^{n}}k\left(r_{i_{1}},r_{i_{2}}\right)\mathbb{E}_{x,x^{\prime}}\left[l\left(x,x^{\prime}\right)\right]\right.\\
        &\left.-\frac{1}{\left(n\right)_{3}}\sum_{i_{3}^{n}}k\left(r_{i_{1}},r_{i_{2}}\right)\mathbb{E}_{x,x^{\prime}}\left[l\left(x,x^{\prime}\right)\right]\right|\\
    	= &\sup_{h\in\mathcal{H}}\left|\vphantom{\sum_{i_{2}^{n}}}\mathbb{E}_{x,r}\left[\vphantom{\sum_{i_{2}^{n}}}\mathbb{E}_{x^{\prime}}\left[l\left(x,x^{\prime}\right)\right]\mathbb{E}_{r^{\prime}}\left[k\left(r,r^{\prime}\right)\right]\right.\right.\\
    	&\left.\left.-\frac{1}{\left(n\right)_{3}}\sum_{i_{3}^{n}}k\left(r_{i_{1}},r_{i_{2}}\right)\mathbb{E}_{x^{\prime}}\left[l\left(x,x^{\prime}\right)\right]\right] \right.\\ &\left.+\frac{1}{\left(n\right)_{3}}\sum_{i_{3}^{n}}k\left(r_{i_{1}},r_{i_{2}}\right)\left(\mathbb{E}_{x,x^{\prime}}\left[l\left(x,x^{\prime}\right)\right]-l\left(x_{i_{2}},x_{i_{3}}\right)\right)\right| \\
    	= &\sup_{h\in\mathcal{H}}\left|\mathbb{E}_{x,r}\left[\frac{1}{\left(n\right)_{3}}\sum_{i_{3}^{n}}\bigg(\mathbb{E}_{r^{\prime}}\left[k\left(r,r^{\prime}\right)\right]\right.\right.\\
    	&\left.\left.-k\left(r_{i_{1}},r_{i_{2}}\right)\bigg)\mathbb{E}_{x^{\prime}}\left[l\left(x,x^{\prime}\right)\right]\vphantom{\sum_{i_{2}^{n}}}\right]\right.\\
    	&\left.+\frac{1}{\left(n\right)_{3}}\sum_{i_{3}^{n}}k\left(r_{i_{1}},r_{i_{2}}\right)\bigg(\mathbb{E}_{x,x^{\prime}}\left[l\left(x,x^{\prime}\right)\right]-l\left(x_{i_{2}},x_{i_{3}}\right)\bigg)\right| \\
    	\le &\sup_{h\in\mathcal{H}}\left|\mathbb{E}_{x,r}\left[\frac{1}{\left(n\right)_{3}}\sum_{i_{3}^{n}}\bigg(\mathbb{E}_{r^{\prime}}\left[k\left(r,r^{\prime}\right)\right]\right.\right.\\
    	&\left.\left.\vphantom{\sum_{i_{2}^{n}}}-k\left(r_{i_{1}},r_{i_{2}}\right)\bigg)\mathbb{E}_{x^{\prime}}\left[l\left(x,x^{\prime}\right)\right]\right]\right|\\
    	&+\sup_{h\in\mathcal{H}}\left|\frac{1}{\left(n\right)_{3}}\sum_{i_{3}^{n}}k\left(r_{i_{1}},r_{i_{2}}\right)\bigg(\mathbb{E}_{x,x^{\prime}}\left[l\left(x,x^{\prime}\right)\right]-l\left(x_{i_{2}},x_{i_{3}}\right)\bigg)\right| \\
    	\le &\sup_{x}\mathbb{E}_{x^{\prime}}\left[l\left(x,x^{\prime}\right)\right]\sup_{h\in\mathcal{H}}\left|\mathbb{E}_{x,r}\left[\mathbb{E}_{r^{\prime}}\left[k\left(r,r^{\prime}\right)\right]\vphantom{\sum_{i_{2}^{n}}}\right.\right.\\
    	&\left.\left.-\frac{1}{\left(n\right)_{3}}\sum_{i_{3}^{n}}k\left(r_{i_{1}},r_{i_{2}}\right)\right]\right| \\
    	&+\sup_{r,r^{\prime}}k\left(r,r^{\prime}\right)\sup_{h\in\mathcal{H}}\left|\mathbb{E}_{x,x^{\prime}}\left[l\left(x,x^{\prime}\right)\right]-\frac{1}{\left(n\right)_{3}}\sum_{i_{3}^{n}}l\left(x_{i_{2}},x_{i_{3}}\right)\right| \\
    	\le &C_1\sup_{h\in\mathcal{H}}\left|\mathbb{E}_{r,r^{\prime}}\left[k\left(r,r^{\prime}\right)\right]-\frac{1}{\left(n\right)_{2}}\sum_{i_{2}^{n}}k\left(r_{i_{1}},r_{i_{2}}\right)\right|\\
    	&+C_2\left|\mathbb{E}_{x,x^{\prime}}\left[l\left(x,x^{\prime}\right)\right]-\frac{1}{\left(n\right)_{2}}\sum_{i_{2}^{n}}l\left(x_{i_{1}},x_{i_{2}}\right)\right|,
        \end{split}
    \end{equation*}
}
where the first inequality follows from properties of $\sup$, and the last inequality follows from the definition of $C_1$ and $C_2$, and the definitions of $(n)_m$.
And finally, the same reasoning can be applied to bound \eqref{term_3}:
{\small
    \begin{equation*}
        \begin{split}
	     &\sup_{h\in\mathcal{H}}\left|\vphantom{\sum_{i_{2}^{n}}}\mathbb{E}_{r,r^{\prime}}\left[k\left(r,r^{\prime}\right)\right]\mathbb{E}_{x,x^{\prime}}\left[l\left(x,x^{\prime}\right)\right]\right.\\
	     &\left.-\frac{1}{\left(n\right)_{4}}\sum_{i_{4}^{n}}k\left(r_{i_{1}},r_{i_{2}}\right)l\left(x_{i_{3}},x_{i_{4}}\right)\right| \\
        = &\sup_{h\in\mathcal{H}}\left|\mathbb{E}_{r,r^{\prime}}\left[k\left(r,r^{\prime}\right)\right]\mathbb{E}_{x,x^{\prime}}\left[l\left(x,x^{\prime}\right)\right]\right.\\
        &\left.-\frac{1}{\left(n\right)_{4}}\sum_{i_{4}^{n}}k\left(r_{i_{1}},r_{i_{2}}\right)l\left(x_{i_{3}},x_{i_{4}}\right)\right.\\
    	&+\left.{}\frac{1}{\left(n\right)_{4}}\sum_{i_{4}^{n}}k\left(r_{i_{1}},r_{i_{2}}\right)\mathbb{E}_{x,x^{\prime}}\left[l\left(x,x^{\prime}\right)\right]\right.\\
    	&\left.-\frac{1}{\left(n\right)_{4}}\sum_{i_{4}^{n}}k\left(r_{i_{1}},r_{i_{2}}\right)\mathbb{E}_{x,x^{\prime}}\left[l\left(x,x^{\prime}\right)\right]\right| \\
    	= &\sup_{h\in\mathcal{H}}\left|\left(\mathbb{E}_{r,r^{\prime}}\left[k\left(r,r^{\prime}\right)\right]-\frac{1}{\left(n\right)_{4}}\sum_{i_{4}^{n}}k\left(r_{i_{1}},r_{i_{2}}\right)\right)\mathbb{E}_{x,x^{\prime}}\left[l\left(x,x^{\prime}\right)\right]\right.\\
    	&\left.+\frac{1}{\left(n\right)_{4}}\sum_{i_{4}^{n}}k\left(r_{i_{1}},r_{i_{2}}\right)\left(\mathbb{E}_{x,x^{\prime}}\left[l\left(x,x^{\prime}\right)\right]-l\left(x_{i_{3}},x_{i_{4}}\right)\right)\right| \\
    	\le &\sup_{h\in\mathcal{H}}\left|\left(\mathbb{E}_{r,r^{\prime}}\left[k\left(r,r^{\prime}\right)\right]-\frac{1}{\left(n\right)_{4}}\sum_{i_{4}^{n}}k\left(r_{i_{1}},r_{i_{2}}\right)\right)\mathbb{E}_{x,x^{\prime}}\left[l\left(x,x^{\prime}\right)\right]\right|\\
    	&+\sup_{h\in\mathcal{H}}\left|\frac{1}{\left(n\right)_{4}}\sum_{i_{4}^{n}}k\left(r_{i_{1}},r_{i_{2}}\right)\left(\mathbb{E}_{x,x^{\prime}}\left[l\left(x,x^{\prime}\right)\right]-l\left(x_{i_{3}},x_{i_{4}}\right)\right)\right| \\
    	\le &C_1\sup_{h\in\mathcal{H}}\left|\left(\mathbb{E}_{r,r^{\prime}}\left[k\left(r,r^{\prime}\right)\right]-\frac{1}{\left(n\right)_{2}}\sum_{i_{2}^{n}}k\left(r_{i_{1}},r_{i_{2}}\right)\right)\right|\\
    	&+C_2\left|\mathbb{E}_{x,x^{\prime}}\left[l\left(x,x^{\prime}\right)\right]-\frac{1}{\left(n\right)_{2}}\sum_{i_{2}^{n}}l\left(x_{i_{1}},x_{i_{2}}\right)\right|.
     \end{split}
    \end{equation*}
    }
\end{proof}

Now, the second term of the RHS of the bound in Lemma \ref{learnability lemma} can be bounded using standard techniques such as Hoeffding's inequality. We therefore shift our attention to the first term. This term can be bounded using Rademacher based techniques.

Let us first recall the definition of the Rademacher complexity of a function class.
\begin{definition}
Let $\mathcal{D}$ be a distribution over $Z$, and let $S=\{z_i\}_{i=1}^n$ be $n$ i.i.d. samples. The empirical Rademacher complexity of a function class $\mathcal{F}$ is defined as:
\begin{equation*}
    \mathcal{R}_n(\mathcal{F}) = \E_\sigma\left[\sup_{f\in \mathcal{F}} \frac{1}{n} \sum_{i=1}^n \sigma_i h(z_i)\right]. 
\end{equation*}
\end{definition}

We assume that $k\left(r,r^{\prime}\right)=s\left(r-r^{\prime}\right)$ for some function $s$ with Lipschitz constant $L_{s}$. Next, we define a distribution over $\mathcal{X}\times\mathcal{X}\times\mathcal{Y}$ by $p^{\prime}\left(\boldsymbol{x}\right)= p^{\prime}\left(\left(x,x^{\prime}\right)\right)=p\left(x\right)p\left(x^{\prime}\right)$, 
and we let
$y\left(x,x^{\prime},\varepsilon,\varepsilon^{\prime}\right)=f\left(x\right)-f\left(x^{\prime}\right)+\varepsilon-\varepsilon^{\prime}$. Now, we can define a new function class $$\mathcal{H}^{2} = \left\{ \left(x_{1},x_{2}\right)\mapsto h\left(x_{1}\right)-h\left(x_{2}\right)\mid h\in\mathcal{H}\right\},$$ 
and consider $$\sup_{h\in\mathcal{H}^{2}}\left|\mathbb{E}_{\boldsymbol{x},y}\ell\left(h\left(\boldsymbol{x}\right),y\right) - \sum_{i=1}^n \ell(h(\textbf{x}_i),y_i)\right|$$ 
where $\ell\left(h\left(\boldsymbol{x}\right),y\right)=s\left(r_{1}-r_{2}\right)$. This is exactly the first term in the bound, which can be bounded using standard generalization bounds. The only missing pieces left are how to relate the Rademacher complexity of $\mathcal{H}^{2}$ to that of $\mathcal{H}$, and how the Lipschitz constant of the residuals' kernel affects it.
\begin{lemma}\label{rademacher_lemma}
$\mathcal{R}_n(\mathcal{H}^2) \le 2\mathcal{R}_n(\mathcal{H})$.
\end{lemma}
\begin{proof}
Let $S=\left\{(z_{i_1},z_{i_2})\right\}_i=1^n$. Then,
\begin{align*}
    \mathcal{R}_n(\mathcal{H}^2) =& \mathbb{E}_\sigma\left[\sup_{h\in\mathcal{H}^2}\frac{1}{n}\sum_{i=1}^n\sigma_ih(s_i)\right] \\
    =&\mathbb{E}_\sigma\left[\sup_{h\in\mathcal{H}}\frac{1}{n}\sum_{i=1}^n\sigma_i\left(h(z_{i_1})-h(z_{i_2})\right)\right] \\
    \le&\mathbb{E}_\sigma\left[\sup_{h\in\mathcal{H}}\frac{1}{n}\sum_{i=1}^n\sigma_ih(z_{i_1})+\sup_{h\in\mathcal{H}}\frac{1}{n}\sum_{i=1}^n\sigma_ih(z_{i_2})\right] \\
    = & 2\mathcal{R}_n(\mathcal{H}),
\end{align*}
where the inequality is due to algebra of $\sup$ and sums.
\end{proof}

As for the Lipschitz constant, a known result relating the Rademacher complexity of a function class to the Rademacher complexity of the class composed with a Lipschitz loss is the following.
\begin{theorem}[\citet{Rademacher_Composition}]\label{lipschitz}
    Let $\ell:\,\mathbb{R}\times\mathcal{Y}\to\mathbb{R}$ be s.t. $\ell(\cdot, y)$ is an $L-$Lipschitz function for all $y$. Denote $\ell\circ\mathcal{H}=\left\{(x,y)\to\ell(h(x),y) \mid h\in \mathcal{H}\right\}$. Then,
    \begin{equation*}
        \mathcal{R}_n(\ell\circ\mathcal{H}) \le L\cdot \mathcal{R}_n(\mathcal{H}).
    \end{equation*}
\end{theorem}
As an example, the following lemma proves the Lipschitz condition for RBF kernels.
\begin{lemma} \label{rbf_lemma}
Assume $\ell(z,y)=\exp(-\gamma(z-y)^2)$, as is the case with RBF kernels, and suppose $|y|\le \frac{M}{2}$ for some $M>0$ for all $y$. Then $\ell(\cdot,y)$ is $\gamma M-$Lipschitz for all $y$. 
\end{lemma}

\begin{proof}
Let $\ell(z,y)=\exp\left(-\gamma\|z-y\|^2\right)$ for some $\gamma>0$, and suppose $\|y\|\le\frac{M}{2}$ for some $M>0$ for all $y\in\mathcal{Y}$. Then,
\begin{align*}
    &\exp\left(-\gamma\|y-z_1\|^2\right) - \exp\left(-\gamma\|y-z_2\|^2\right) \\
    =&\exp\left(-\gamma\|y-z_1\|^2\right)\left(1-\exp\left(\gamma\left(\|y-z_1\|^2-\|y-z_2\|^2\right)\right)\right) \\
    \le & \exp\left(-\gamma\|y-z_1\|^2\right)\left(1-\left(1+\gamma\left(\|y-z_1\|^2-\|y-z_2\|^2\right)\right)\right) \\
    =&\exp\left(-\gamma\|y-z_1\|^2\right)\left(\gamma\left(\|y-z_2\|^2-\|y-z_1\|^2\right)\right) \\
    \le & \gamma\left(\|y-z_2\|^2-\|y-z_1\|^2 \right) \\
    \le & \gamma\|z_1-z_2\|^2\\
    \le &\gamma \cdot M\|z_1-z_2\|
\end{align*}
where the first inequality is due to the fact that $1+x\le e^x$, the second is due to the fact that $\exp(-c) \le 1 \, \forall c\ge0$, the third is due to triangle inequality and the last is due to the definition of $M$.
\end{proof}
Before concluding, we state the  uniform convergence result based on the Rademacher complexity of a class.
\begin{theorem}[\citet{mohri2018foundations}, Ch. 3]
    Suppose $f(z)\in[0,1]$ for all $f\in \mathcal{F}$, and let $\delta \in (0,1)$. Then, with probability of at least $ 1 - \delta$ over the choice of $S$, the following holds uniformly for all $f \in \mathcal{F}$:
    \begin{equation*}
        \left|\E_\mathcal{P}[f(z)] - \frac{1}{n}\sum_{i=1}^n f(z_i)\right| \le 2\mathcal{R}_n(\mathcal{F}) + O\left(\sqrt{\frac{\ln(1/\delta)}{n}}\right).
    \end{equation*}
    \label{rademacher_bound}
\end{theorem}
Equipped with those results, the learnability Theorem \ref{leanability_thm} immediately follows. 
\begin{proof}[Proof of Theorem \ref{leanability_thm}]
This is a direct application of the previous lemmas and Hoeffding's inequality.
\end{proof}

\subsection{Experiments' details}
\subsubsection{Synthetic data}
In order to find the $l_2$ regularization parameter, we perform cross validation on validation data created from $10\%$ of the training set, where the possible values are in $\{15, 12, 10, 5, 1, 0.1, 0.01, 0.001, 0.0001, 0.00001, 0.000001,0\}$, for absolute- and HSIC-losses, and $\{35,37,...,69\}$ for the squared-loss. 

When training with absolute-loss, we used stochastic gradient descent, with initial learning rate determined by cross validation from $\{0.05, 0.01, 0.005, 0.001, 0.0005, 0.0001,0.00005,0.00001\}$, and later decayed using inverse scaling. When training with HSIC-loss, the learning rate was drawn from a uniform distribution over $[0.0001,0.0002]$. 

\subsubsection{Rotated MNIST}
\begin{wrapfigure}{r}{0.17\textwidth}
\centering
\includegraphics[scale=0.22]{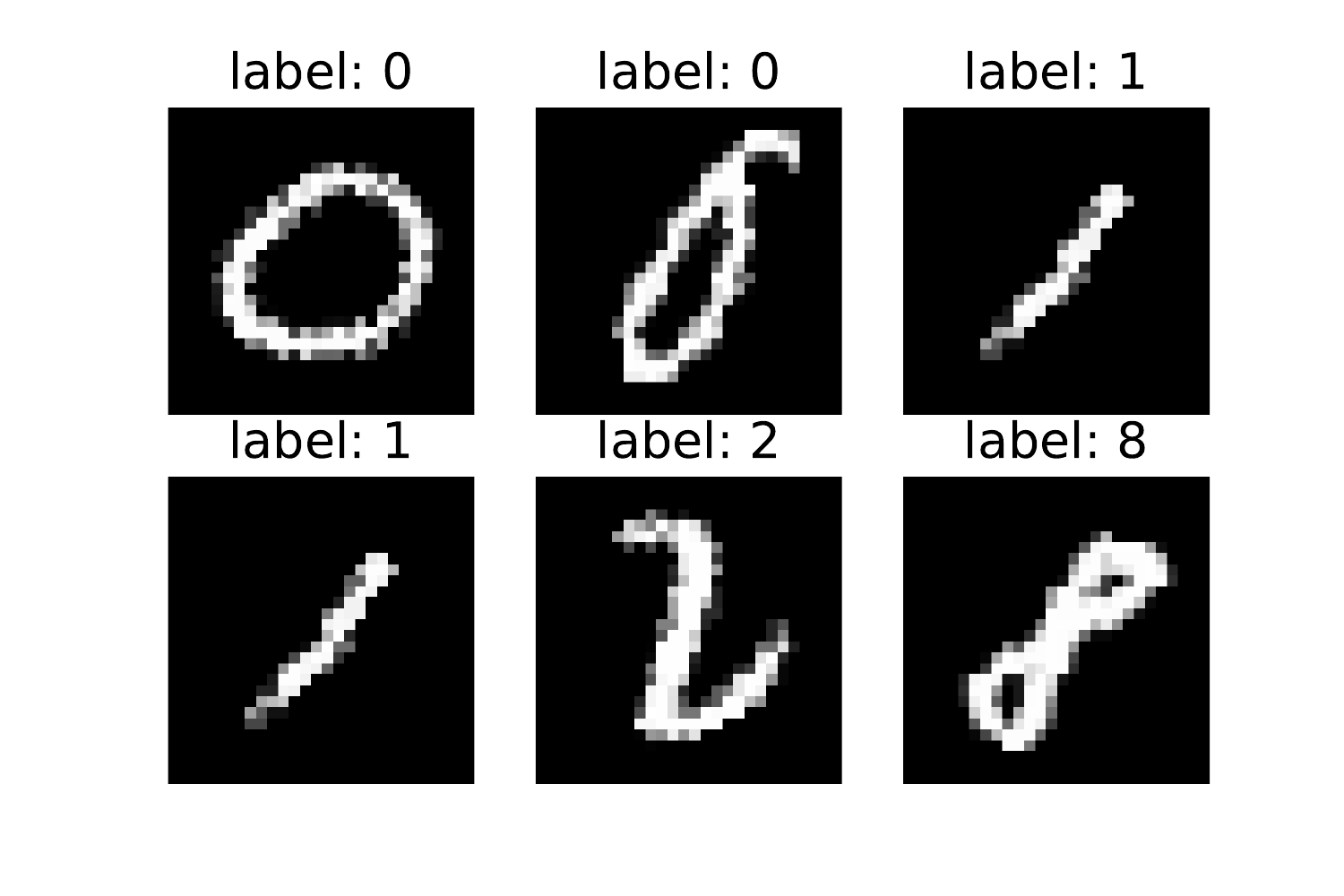}
\caption{MNIST \textsc{target} images.}
\label{rotated_mnist_images}
\end{wrapfigure}
Both losses were optimized using Adam \citep{Kingma2015AdamAM}, and the learning rate was drawn each time from a uniform distribution over $[10^{-4}, 4\cdot 10^{-4}]$. Experimenting with different regimes of the learning rate gave similar results.

\end{document}